\documentclass[11pt]{article}
\usepackage[margin=1in]{geometry}
\usepackage{setspace}
\usepackage[utf8]{inputenc}
\usepackage[backref,colorlinks,citecolor=blue,bookmarks=true]{hyperref}
\usepackage{mathtools, amssymb, bbm}
\usepackage[capitalize]{cleveref}
\usepackage{enumerate}
\usepackage{todonotes}
\usepackage{bm}
\usepackage{amsthm}
\usepackage{algorithm}
\usepackage{algorithmic}

\setuptodonotes{inline}
\usepackage[font=small]{caption}

\newcommand{\mD}{\mathcal{D}}

\newcommand{\mA}{\mathcal{A}}
\newcommand{\X}{\mathcal{X}}

\newcommand{\zo}{\ensuremath{\{0,1\}}}

\newcommand{\lossPred}{\mathsf{LP}_{\ell, p}}
\newcommand{\clp}{\mathsf{SEP}_{\ell, p}}
\newcommand{\sqLoss}{\mathsf{sq}}

\newcommand{\adv}{\mathsf{adv}}
\newcommand{\inp}{i} %

\newcommand{\eat}[1]{}

\newcommand{\CP}[1]{\textcolor{red}{CP: {\em #1}}}
\newcommand{\PG}[1]{\textcolor{blue}{PG: {\em #1}}\\}
\newcommand{\AG}[1]{\textcolor{teal}{AG: {\em #1}}\\}
\newcommand{\UW}[1]{\textcolor{purple}{UW: {\em #1}}\\}
\newcommand{\AK}[1]{\textcolor{pink}{AK: {\em #1}}\\}

\eat{
\newcommand{\UW}[1]{}
\newcommand{\AK}[1]{}
\newcommand{\AG}[1]{}
\newcommand{\CP}[1]{}
\newcommand{\PG}[1]{}
}

\newcommand{\ty}{\tilde{y}}

\newcommand*{\R}{{\mathbb{R}}}
\newcommand*{\calC}{{\mathcal{C}}}

\newcommand*{\calH}{{\mathcal{H}}}
\newcommand*{\calG}{{\mathcal{G}}}
\newcommand*{\calD}{{\mathcal{D}}}
\newcommand*{\calF}{{\mathcal{F}}}

\newcommand*{\calX}{{\mathcal{X}}}

\newcommand*{\calB}{{\mathcal{B}}}
\newcommand*{\calA}{{\mathcal{A}}}

\newcommand*{\calL}{{\mathcal{L}}}
\newcommand*{\calU}{{\mathcal{U}}}

\let\phi\varphi

\DeclareMathOperator*{\ex}{\mathbb{E}}
\DeclareMathOperator*{\E}{\mathbb{E}}

\DeclareMathOperator*{\argmin}{arg\,min}

\DeclareMathOperator{\sgn}{sign}

\newcommand{\ignore}[1]{}

\newcommand{\pred}{p}
\newcommand{\pbayes}{\pred^*}

\DeclareMathOperator{\MCE}{MCE}

\DeclareMathOperator{\Ber}{Ber}

\DeclarePairedDelimiterX{\divx}[2]{(}{)}{%
  #1\;\delimsize\|\;#2%
}

\theoremstyle{plain}
\newtheorem{theorem}{Theorem}[section]
\newtheorem{lemma}[theorem]{Lemma}
\newtheorem{corollary}[theorem]{Corollary}

\newtheorem{claim}[theorem]{Claim}

\theoremstyle{definition}
\newtheorem{definition}[theorem]{Definition}

\theoremstyle{remark}

\numberwithin{equation}{section}

\title{When does a predictor know its own loss?}
\author{Aravind Gollakota\\
Apple\\
\and
Parikshit Gopalan\\
Apple\\
\and
Aayush Karan\\
Harvard University\thanks{Work done  while interning at Apple. Authors in alphabetical order.}\\
\and 
Charlotte Peale\\
Stanford University\footnotemark[1]\\
\and
Udi Wieder\\
Apple
}

\begin{document}
\date{February 18, 2025}

\maketitle
\begin{abstract}

    Given a predictor and a loss function, how well can we predict the loss that the predictor will incur on an input? This is the problem of loss prediction, a key computational task associated with uncertainty estimation for a predictor. In a classification setting, a predictor will typically predict a distribution over labels and hence have its own estimate of the loss that it will incur, given by the entropy of the predicted distribution. Should we trust this estimate? In other words, when does the predictor know what it knows and what it does not know?

    In this work we study the theoretical foundations of loss prediction.
    Our main contribution is to establish tight connections between nontrivial loss prediction and certain forms of multicalibration \cite{hebert2018multicalibration}, a multigroup fairness notion that asks for calibrated predictions across computationally identifiable subgroups.
    Formally, we show that a loss predictor that is able to improve on the self-estimate of a predictor yields a witness to a failure of multicalibration, and vice versa. This has the implication that nontrivial loss prediction is in effect no easier or harder than auditing for multicalibration. We support our theoretical results with experiments that show a robust positive correlation between the multicalibration error of a predictor and the efficacy of training a loss predictor. 
\end{abstract}

\section{Introduction}
It is increasingly common for large machine learning models to be part of a pipeline where a base model is trained by a provider that has access to large-scale data and computational power, and the model is then deployed by a heterogeneous set of downstream consumers, for a diverse range of prediction tasks. Not only could the tasks be very different from each other, they might involve data distributions, loss functions and other metrics and features that are markedly different from those used to train the model. Indeed, often the data sources used in training are not even disclosed to the downstream application. A typical instantiation of this framework is zero-shot classification, where (say) an LLM is required to classify texts into classes described by the user. Another important case is that of medical classification where the base model was trained on one set of features, say lab reports, but the model (or human) downstream has access to additional features such as patient history.  

In such a situation, a user might want to delve deeper into how the model is likely to perform on their specific task. They might seek to discover  problematic regions of the input space where the model performs poorly, where performance is measured by an appropriate loss function chosen by the user. This information could prove valuable in several ways:
\begin{enumerate}
    \item Active and continual learning: Users could address performance issues by collecting additional data points from problematic regions and fine-tune the model on this enhanced dataset.
    \item Fairness considerations: The analysis might reveal potential biases or inequities in the model's performance across different subgroups.
    \item Selective prediction: Such insights could guide downstream users on when to rely on the model's predictions and when to exercise caution. In cases where predictions are likely to be unreliable, users might opt to consult external experts or alternative models instead.
\end{enumerate}
By systematically identifying and addressing these performance vulnerabilities, users can judge the model's reliability, fairness, and overall utility. A provider who desires to improve their model would similarly benefit from knowing where their model performs poorly.

This discussion motivates the problem of loss prediction, which we now define. In this work we focus for concreteness on the binary classification setting, although many of the results extend to multiclass classification as well. We are given a pre-trained predictor $p:\X \to [0,1]$ (where $\calX$ denotes the space of inputs), a target loss $\ell: \zo \times [0,1] \to \R$\footnote{It will be convenient to assume that this loss is \emph{proper}, namely that for any $p^*$ the expected true loss $\ex_{y \sim p^*}[\ell(y, p)]$ is minimized at $p = p^*$. Canonical examples are the cross-entropy and the squared loss. The case of general losses can be reduced to that of proper losses.}, and some labeled data $(x, y)$ drawn from an unknown distribution $\calD$ on $\calX \times \{0, 1\}$. The goal is to estimate the loss $\ell(y, p(x))$ incurred at a point $x$ using a loss predictor $\lossPred : \calX \to \R$. This can be viewed as a regression problem, and we measure the quality of a loss predictor by its expected squared loss with respect to the true loss, i.e.\ $\ex_{(x, y) \sim \calD}[\big(\lossPred(x) - \ell(y, p(x)) \big)^2]$.

Loss prediction is closely connected to the well-studied problem of uncertainty estimation. A standard measure of predictive uncertainty at a point is the expected loss that a predictor suffers at that point \cite{kull2015novel}, and estimating this requires solving the problem of loss prediction. Given such a loss predictor, its uncertainty estimate is then often decomposed into two parts: aleatoric uncertainty, which is the uncertainty stemming from the randomness in nature, and epistemic uncertainty, which is the uncertainty arising from shortcomings in our model and/or training data.\footnote{There are many proposals for how to achieve such a decomposition, see e.g.\ \cite{hullermeier2021aleatoric}, not all of which come with rigorous guarantees. Recent work of \cite{ahdritz2024provable} does give rigorous guarantees, but it requires an enhancement to the standard learning model called learning with snapshots. See also the discussion of related work therein.
}
Since epistemic uncertainty can be driven down with more data and fine tuning, active learning strategies have been proposed that use loss predictors to decide what regions to prioritize for collecting more data \cite{yoo2019learning,lahlou2021deup}. Loss predictions can also be used for various other applications, including deciding when a model should abstain from learning, or route the input to a stronger model.
Consequently, there has been plenty of applied work on the problem of loss prediction, but little theoretical analysis (see \cref{sec:related} for more discussion of related work).

\subsection{Our contributions}
In this work, we initiate a study of the theoretical foundations of loss prediction. We formalize the task of loss prediction and connect it to the basic primitives of computational learning.

\paragraph{The self-entropy predictor of loss.}
The first question is what baseline one should use to measure the quality a loss predictor. Drawing from work on outcome indistinguishability \cite{OI, OI2}, we propose a baseline based on the fact that a predictor posits a certain model of nature: that labels for $x$ are drawn according to a Bernoulli distribution with parameter $p(x)$. This entails a belief about the expected loss it will incur at a point. In the case of squared loss, this estimate is $p(x)(1 - p(x))$ at the point $x$.  By results of \cite{gneiting2007strictly}, for any proper loss $\ell$, there exists a concave ``generalized entropy'' function $H_{\ell}:[0,1] \to \R$ such that the prediction is $H_{\ell}(p(x))$. We refer to this as the self-entropy predictor. Using this as our baseline, we ask when it is possible for a loss predictor to do better than the self-entropy predictor. At a high level, we wish to understand

\begin{center}
{\em When can a loss-predictor beat a model in estimating what the model knows and does not know?} 
\end{center}

\paragraph{A hierarchy of loss prediction models.} It is natural that loss predictors should receive the input features $\inp(x)$\footnote{For clarity, we make a distinction between the abstract input $x$ (e.g., an individual) and its input feature representation $\inp(x)$ (e.g., features collected about the individual).} and the prediction $p(x)$ as inputs.  But this does not capture some important architectures for loss prediction that are used in practice; for instance the works of \cite{yoo2019learning, kirillov2023segment} which consider models that can access representations of $x$ that are computed by the neural network computing $p$. Accordingly, we model loss predictors as taking inputs $\phi(p, x)$ lying in an abstract feature space and returning a loss prediction $\lossPred(p(x))$. We define a hierarchy of loss predictors of increasing strength, depending on expressivity of $\phi(p,x)$ (Definition~\ref{def:lp}):
    \begin{enumerate}
        \item \emph{Prediction-only loss predictors} only have access to $p$'s prediction at a point $x$, i.e. $\phi(p, x) = p(x)$. The self-entropy predictor of loss is an example.
        \item \emph{Input-aware loss predictors} have additional access to the input features $\inp(x)$, i.e. $\phi(p, x) = (p(x), \inp(x))$. 
        \item \emph{Representation-aware loss predictors} have access to $\phi(p, x) = (p(x), \inp(x), r(x))$, where $r(x)$ is some representation of $x$. In this case, we further distinguish between two settings:
        \begin{itemize}
            \item Internal representations $r(x) = r_p(x)$ where $r_p(x)$ is computed by $p$ in the course of computing $p(x)$. 
            \item External representations $r(x) = r_e(x)$ which are not explicitly computed by $p$. 
            \end{itemize}
    \end{enumerate}
Internal representations could for instance correspond to the embedding produced by the last few layers of a deep neural net. External representations could be the representation of $x$ obtained from a different model, or additional features added by consulting human experts. The related work in Section~\ref{sec:related} gives examples of both kinds of representations that have been considered in the literature. 

Finally, we define the advantage of a loss predictor over the self-entropy predictor to be the difference in the squared loss incurred by the two loss predictors (Definition~\ref{def:lp-advantage}). 

\paragraph{Relation to auditing for multicalibration.}
Multicalibration is a multigroup fairness notion introduced by \cite{hebert2018multicalibration}, which has since found numerous other applications \cite{kim2022universal, OI, omni}. We show that learning a loss predictor with a non-trivial advantage is tightly connected to auditing the predictor for multicalibration. 
 At a high level, we show the following correspondence, which we formalize in Theorem~\ref{thm:mc-conv-vs-loss-pred}:

\bigskip

\begin{center}
\begin{minipage}{0.41\textwidth}
Finding a \textbf{prediction-only} loss predictor with good advantage
\end{minipage}
\quad
$\Leftrightarrow$
\quad
\begin{minipage}{0.41\textwidth}
Identifying a \textbf{calibration} violation for 
$p$
\end{minipage}
\end{center}

\bigskip

\begin{center}
\begin{minipage}{0.41\textwidth}
Finding an \textbf{input-aware} loss predictor with good advantage
\end{minipage}
\quad
$\Leftrightarrow$
\quad
\begin{minipage}{0.41\textwidth}
Identifying a \textbf{multicalibration} violation for 
$p$
\end{minipage}
\end{center}

\bigskip

\begin{center}
\begin{minipage}{0.41\textwidth}
Finding a \textbf{representation-aware} loss predictor with good advantage
\end{minipage}
\quad
$\Leftrightarrow$
\quad
\begin{minipage}{0.41\textwidth}
Identifying a \textbf{representation-aware multicalibration} violation for 
$p$, where the auditor function is of the form $c(\phi(p,x))$. 
\end{minipage}
\end{center}

\bigskip

In all cases, the regions where the multicalibration violations occurs arise from analyzing where the loss predictor and the self-entropy predictor differ from each other. The first two notions in our hierarchy, calibration and multicalibration, have been extensively studied in previous works \cite{FV99, hebert2018multicalibration}. The last member of the hierarchy, representation-aware multicalibration, is a strengthening of multicalibration that naturally extends the multicalibration framework.  

Furthermore, we explore how the lens of multicalibration proves valuable in predicting well for a large class of losses, particularly when learning individual predictors for each loss is impractical. In Theorem~\ref{thm:1-lip-mc-pred}, we show that via standard techniques for learning multicalibrated predictors, we can efficiently learn a predictor whose self-entropy predictions for every 1-Lipschitz proper loss (of which there are infinitely many) are comparable to the best-in-class loss predictor for each loss from some fixed class of candidate predictors. 

\paragraph{On calibration blind-spots for loss prediction.}
Calibration is not necessary for producing good estimates of the true loss. For instance, a predictor that predicts $p(x) = 1/2$ on every input will indeed incur a squared loss of $1/4$, matching its self-entropy predictor regardless of the true labels. But depending on the distribution of labels, this predictor might be very far from calibrated, and need not even be accurate in expectation. 

Our results imply a simple characterization of such ``blind spots'' for any proper loss $\ell$ as points $p$ where $H_{\ell}'(p) = 0$.\footnote{Recall that $H_{\ell}(p)$ is the concave entropy function corresponding to $\ell$.} In terms of the loss $\ell$, this is equivalent to $\ell(0, p) = \ell(1,p)$, so that the loss incurred is independent of the label, and hence predicting the expected loss for such $p$  is trivial. For strictly proper losses, the function $H_{\ell}$ is strictly concave, and there is a unique point where this happens.

This introduces some subtlety in the type of multicalibration violations that arise from our correspondence; the standard calibration error at $p$ is weighted by a factor of $H_{\ell}'(p)$. Hence non-trivial loss prediction corresponds to  (multi)calibration violations at prediction values $p$ such that $H_{\ell}'(p)$ is far from $0$. 

\paragraph{Experimental results.} We empirically verify that there is a correspondence between loss prediction and multicalibration (see \cref{sec:exp}). Focusing on input-aware loss prediction algorithms run across a variety of base predictor types, we find that: 
\begin{itemize} 
\item As the multicalibration error of the base model increases, the advantage of the loss prediction over the self estimate of the loss increases. 
\item Loss predictors are more advantageous on data subgroups that have higher calibration error. 
\end{itemize} 
Our experiments suggest that regression-based loss predictors present an effective way to audit for multicalibration and are an intriguing avenue towards developing efficient multicalibration algorithms for practice.

\subsection{Takeaways from our result}

The main takeaway from our work is that non-trivial loss prediction is no easier (and not much harder) than auditing the predictor itself. Any predictor that improves over the self-entropy predictor could be used to find (and possibly fix) multicalibration issues in the predictor. 

\paragraph{Practical multicalibration using loss prediction.} The complexity of multicalibration depends crucially on the class of test functions used.  For complex functions, our equivalence suggests a novel approach to multicalibration auditing: choose a proper loss, run a regression for loss prediction, and see if the loss predictor outperforms the self-entropy predictor. This is a simple and practical approach that is able to leverage the strength of any well-engineered library for regression. In our experiments we show that this is indeed effective, with loss prediction advantage being robustly correlated with multicalibration error across mutliple base predictors as well as subgroups.\footnote{To get around the blind-spot issue, one could choose a few strictly proper loss functions each with a different blind spot. This is easy to do, given the correspondence between convex functions and proper losses \cite{gneiting2007strictly}.}

\paragraph{On two-headed architectures.}

There has been work on training deep neural nets with two heads: a prediction head $p$, and a loss prediction head $\lossPred$ \cite{yoo2019learning, kirillov2023segment}. The loss prediction head has access to the embedding of the inputs produced by the last few layers of the neural net, and can be modeled by a representation-aware loss predictor of low complexity. Our result shows that (at least in a classification setting) one of the following must be true:
\begin{itemize}
    \item The loss prediction head does not give much advantage over the self-entropy predictor, which only requires prediction access.
    \item The prediction head is not optimal, as evidenced by a multicalibration violation witnessed by the difference between the loss prediction head and the self-entropy predictor (see Lemma~\ref{lem:adv-implies-mc-err}).
\end{itemize}
The ideal situation for a well-trained model is clearly the former. 

Note that this does not mean that two-headed architectures are not useful: the two heads may influence the training dynamics in a subtle way, with the loss-predictor head revealing complex regions where multicalibration fails. However, what our result implies is that when training concludes, we want to be in the situation where the loss-predictor is not much better than the self-entropy predictor. This is analogous to the situation with GANs \cite{GANs}, where at the end of training, we would ideally like the generator to be able to fool the discriminator. But in the intermediate stages of training, the discriminator helps improve the quality of the generator. 

Two-headed architectures may also be useful in prediction problems more general than ordinary classification, such as image segmentation \cite{kirillov2023segment}, where a predictor does not necessarily come with a self-entropy estimate at all.

\paragraph{Extra information helps: when loss prediction might be effective.}

An important scenario is where the loss-predictor may have informative features $\phi'(x)$ about the input $x$ that were not available to the entity that was training the model $p$. For example, consider a neural net that is trained to screen X-rays for prevalence of a certain medical condition. Such models may be trained by aggregating data from across several hospitals. A hospital that is trying to use this model might not have the same computational resources available to them. But they might have access to other useful information such as observations made by a doctor or the patient's medical history. 

In such a case, even a model which is multicalibrated with respect to complex functions over the features $\phi(p, x)$ might not be multicalibrated with respect to simple functions over a new set of features. This was illustrated in the recent work of \cite{alur24} in their work on incorporating human judgments to improve on model predictions.  Another natural scenario in which the loss predictor may have extra information is if it uses a powerful pretrained foundation model. The work of \cite{jain2022distilling} does precisely this, leveraging embeddings from CLIP \cite{radford2021learning}. In such settings, improvements to the predictor, and loss predictors with a non-trivial advantage are both possible.

\paragraph{Organization.} We define loss predictors in Section \ref{sec:lp} and recall the relevant notions of multicalibration in Section \ref{sec:mc}. We present the equivalence between loss prediction with an advantage over the self-entropy predictor and multicalibration in Section \ref{sec:lp-mc}. We discuss how to efficiently find predictors that give good self-entropy predictors for multiple loss functions in Section \ref{sec:multiple-loss}. In Section \ref{sec:exp}, we empirically demonstrate the correspondence between loss prediction advantage and multicalibration violations, and show that it holds across multiple architectures and data subgroups. We discuss related work in detail in Section \ref{sec:related}. In Appendix \ref{sec:non-proper}, we present the extension of our results to the case where the losses are non-proper. 

\section{Loss prediction}
\label{sec:lp}

We consider binary classification, with a distribution $\mD$ on  $\X \times \zo$. 
A predictor is a function $p:\X \to [0,1]$. The Bayes optimal predictor is defined as $p^*(x) = \ex[y|x]$. Given $p$, we define the simulated distribution $\mD(p)$ on $\X \times \zo$ where $x$ is drawn as in $\mD$, and $y|x \sim \Ber(p(x))$. Let $\ell: \zo \times [0,1] \to [0,1]$ be a proper loss function.\footnote{ The case of general losses reduces to the proper loss case; please see \cref{sec:non-proper} for details. We also assume for technical convenience that the loss is bounded. Losses that are not strictly bounded, such as cross entropy, can be handled with some further care and constraints on predicted probabilities.} We will use the following characterization of proper losses.

\begin{lemma}[\cite{gneiting2007strictly}] For every proper loss $\ell$,  there exists a concave function $H_{\ell}: [0, 1] \rightarrow \R$ so that
\[ \ell(y, v) = H_{\ell}(v) + (y - v)H_{\ell}'(v).\]
where $H_{\ell}'(v)$ is a ``superderivative'' of $H_{\ell}$, i.e. for any $v, w \in [0, 1]$, $H_{\ell}(v) \leq H_{\ell}(w) + (v - w)H_{\ell}'(w)$. 
\end{lemma}
When $H_{\ell}(v)$ is differentiable at all $v \in [0, 1]$, the superderivative is unique, and is just the derivative.  From the definition it follows that 
\begin{align*} 
H_{\ell}(v) &= \ex_{y \sim \Ber(v)}[\ell(y,v)] \in [0,1]\\
H_{\ell}'(v) &= \ell(1, v) - \ell(0,v) \in [-1,1]
\end{align*}
Let $L(p^*; p) = \ex_{y \sim \Ber(p^*)}[\ell(y, p)]$ denote the expected loss when $y \sim \Ber(p^*)$ but we predict $p$. Then
\begin{align}
\label{eq:exp-loss}
 L(p^*;p) = H_{\ell}(p) + (p^* - p)H_{\ell}'(p) \geq H_{\ell}(p^*) = L(p^*; p^*) 
\end{align}
where the inequality follows from the concavity of $H_{\ell}$, and is equivalent to the loss $\ell$ being proper.

We now define the notion of a loss predictor. 

\begin{definition}[Loss predictor]
\label{def:lp}
    Let $p$ be a predictor and $\ell$ be a proper loss. Let $\Phi$ be an abstract feature space, which we will make concrete shortly. 
    A \emph{loss predictor} is a function $\lossPred: \Phi \to \R$, which takes as input some features $\phi(p,x) \in \Phi$ related to a point $x$ and its prediction using $p$, and returns an estimate $\lossPred(\phi(p,x))$ of the expected loss $\ex[\ell(y, p(x))|x]$ suffered by $p$ at the point $x$. We define a hierarchy of loss predictors of increasing strength, depending on the information contained in $\phi(p,x)$:
    \begin{enumerate}
        \item \emph{Prediction-only loss predictors} only have access to $p$'s prediction at a point $x$, i.e. $\phi(p, x) = p(x)$. The most natural choice for a prediction-only loss predictor is given by the self-entropy predictor, which we will define in Definition~\ref{def:can-loss-pred}.
        \item \emph{Input-aware loss predictors} have access to the input features $\inp(x)$ used to train the model $p$, as well as its prediction, i.e. $\phi(p, x) = (\inp(x), p(x))$. 
        \item \emph{Representation-aware loss predictors} have access to $\phi(p, x) = ( p(x), \inp(x), r(x))$, where $r(x)$ is some representation of $x$. We distinguish between two kinds of representations:
        \begin{itemize}
            \item Internal representations: The representation $r(x)= r_p(x)$ consists of features that are explicitly computed by the predictor $p$ in the course of computing $p(x)$. For instance, they could consist of the embedding of $x$ produced by the last few layers of a DNN.  
            \item External representations: The representation $r(x) = r_e(x)$ consists of features that are not explicitly computed by the  predictor $p$. For instance, they could be the representation of $x$ obtained from a different model, or by consulting human experts.         
        \end{itemize}
     \end{enumerate}
\end{definition}

A few comments on the definition: 

\begin{itemize}
\item Two-headed architectures that simultaneously train both the predictor and the loss-predictor (such as \cite{yoo2019learning,kirillov2023segment}) are a class of internal representation-aware predictors. In contrast, loss-predictors that use an embedding produced by a foundation model (such as \cite{jain2022distilling}, which audits the errors of the predictor) are external representation-aware.  

\item   If we allow the loss predictor to be significantly more complex than the predictor $p$, then it could compute $r_p(x)$ from $\inp(x)$ using the model $p$. So the gap between input-aware and representation-aware loss predictors diminishes as the loss-predictor becomes more computationally powerful. But in the (important) setting where the loss predictor is less computationally powerful than the predictor, there could be a gap.

\item In contrast, external representations might contain auxiliary information  that cannot be computed using $\inp(x)$, regardless of the computational power of the loss predictor. 
\end{itemize}

The loss predictor can be trained using standard regression, given access to a training set of points $(\phi(p, x), y)$ where $(x, y)$ are drawn from the distribution $\mD$.  One can measure the performance of our loss predictor as we would with any regression problem. We formulate it using the squared loss, as $\ex[(\ell(y, p(x)) - \lossPred(\phi(p, x))^2]$. It follows from Equation \eqref{eq:exp-loss} that the Bayes optimal loss predictor is given by
$\lossPred^*(x) = L(p^*(x); p(x))$. But computing this requires knowing the Bayes optimal predictor $p^*$, and is likely to be infeasible in most settings. Rather, we will compare our loss predictor to a canonical baseline which we describe next.

\paragraph{The self-entropy predictor.}

Following \cite{OI}, given a predictor $p$, we define the simulated distribution $\mD(p)$  on pairs $(x, \ty) \in \X \times \zo$,  where $x \sim \mD$ and $\ex[\ty|x] = p(x)$. The predictor hypothesizes that this how labels are being generated. Hence for each $x \in \X$, the self-entropy predictor predicts the expected loss according to this distribution. 

\begin{definition}[Self-entropy predictor]\label{def:can-loss-pred}
    Given a proper loss $\ell$ and predictor $\pred$,  the \emph{self-entropy predictor} is the prediction-only loss predictor $\clp: [0, 1] \to \R$  that predicts the expected loss when $\ty \sim \Ber(\pred(x))$ at each $x$; that is
    \[\clp(\pred(x)) = \ex_{\ty \sim \Ber(p(x))}[\ell(\ty, p(x))] = H_{\ell}(\pred(x)).\]
\end{definition}

We use the self-entropy predictor as our baseline. Hence the question is when can we learn a loss predictor with significantly lower squared loss than the self-entropy predictor. We formalize this using the notion of advantage of a loss predictor over the self-entropy predictor.

\begin{definition}[Advantage of a loss predictor]\label{def:lp-advantage}
    Define the advantage of a loss predictor $\lossPred$ over the self-entropy predictor to be the difference in the squared error
    \[ \adv(\lossPred) = \ex[(\ell(y, p(x)) - \clp(p(x)))^2] - \ex[(\ell(y, p(x)) - \lossPred(\phi(p, x))^2]. \]
\end{definition}%

We want loss predictors whose advantage is positive and as large as possible.
Our goal is understand under what conditions we can hope to learn such a predictor. 

\paragraph{On non-proper losses.} So far we have assumed that we a trying to predict the proper loss incurred by a predictor. We can generalize this to a setting where we have a hypothesis $h:\X \to \mA$ (for instance $h$ might be a binary classifier), and a loss function $\ell:\zo \times \mA \to \R$. It turns out that our theory extends seamlessly to the non-proper setting, under rather mild assumptions on the hypothesis $h$. We present this extension in Appendix \ref{sec:non-proper}.

\section{Multicalibration}
\label{sec:mc}

Having defined our notion of a loss predictor, we next introduce the framework of multicalibration proposed by~\cite{hebert2018multicalibration}. Our definition is most similar to the presentation used in~\cite{kim2022universal}.

\begin{definition}[Multicalibration]
\label{def:mc}
    Let $\phi(p, x) \in \Phi$ be some auxiliary set of features related to the computation of $p(x)$, which we define concretely below. Let $\calC$ be a class of weight functions $c: \Phi \rightarrow [-1, 1]$, and $p: \calX \rightarrow [0, 1]$ a binary predictor for a target distribution $\calD$ over $\calX \times \{0, 1\}$. Then, the multicalibration error of $p$ with respect to $\calC$ is defined as 
    \[\MCE(\calC, p) := \max_{c \in \calC} \left|\ex_{x, y \sim \calD}[(y - p(x))c(\phi(p, x))]\right|.\]
    The information contained in $\phi(p, x)$ gives rise to a hierarchy of multicalibration notions of increasing strength:
    \begin{enumerate}
        \item \emph{Calibration} corresponds to the setting where  $\phi(p, x)= p(x)$, and test functions can only depend on $p$'s prediction. 
        \item \emph{Multicalibration} corresponds to the case where test functions can additionally depend on the input features, i.e. $\phi(p, x) = (p(x), \inp(x))$.
        \item \emph{Representation-aware multicalibration} is a strengthening of multicalibration where test functions can additionally depend on some representation $r(x)$ of $x$ i.e., $\phi(p, x) = (p(x), \inp(x), r(x))$. We distinguish between internal representations $r_p(x)$ and external representations $r_e(x)$ as with loss predictors (Definition \ref{def:lp}).
    \end{enumerate}
\end{definition}

The first two levels in this hierarchy, calibration and multicalibration, have been extensively studied in previous works. 
In standard multicalibration, we require that a predictor $p(x)$ be well-calibrated under a broad class of test functions, $\calC$, that depend only on $\inp(x)$ and $p(x)$. The literature on multicalibration typically identifies $\inp(x)$ with $x$ itself. The last level of the hierarchy, representation-aware multicalibration, is a strengthening of multicalibration that naturally extends the multicalibration framework of ~\cite{hebert2018multicalibration}. As in the case of loss-predictors, the gap between internal representations $r_p(x)$ and $\inp(x)$ is computational; whereas the gap between external representations $r(x)$ and $\inp(x)$ could be information-theoretic.

\begin{definition}[Multicalibration violation witness]
    We say that a function $c: \Phi \times [0,1] \rightarrow [-1, 1]$ is a witness for a multicalibration violation of magnitude $\alpha$ for a predictor $p$ if 
    \[\left|\ex_{x, y \sim \calD}[(y - p(x))c(\phi(p, x))]\right| > \alpha.\]
\end{definition}

\cite{hebert2018multicalibration} showed that if we find such a witness, we can use it to improve the predictor $p$ in a way that reduces the squared loss. While their argument is stated for the input-aware setting where $\phi(p,x) = (p(x), \inp(x))$, it applies to the representation-aware setting as well.

\section{Loss prediction advantage and multicalibration auditing}
\label{sec:lp-mc}

In this section, we establish the relationship between learning loss predictors with good advantage, and auditing for multicalibration, i.e. finding a $c$ that witnesses a large multicalibration violation. The main result of our section is the following theorem, which establishes the correspondence between various levels of loss predictors and multicalibration requirements, when instantiated with the appropriate values $\phi(p, x)$:

\begin{theorem}\label{thm:mc-conv-vs-loss-pred}
    Let $\calF$ be a class of loss predictors $f: \Phi \rightarrow [0, 1]$. 
    Let $\calF' \supseteq \calF$ be the augmented function class defined as 
    \[\calF' = \{\Pi_{[0,1]}((1 - \beta)H_{\ell}(p(x)) + \beta f(\phi(p,x))) : \beta \in [-1, 1], f \in \calF\}.\]
    Let $\calC$ be a class of weight functions defined as 
    \[\calC = \{ (f(\phi(p, x)) - H_{\ell}(p(x)))H_{\ell}'(p(x)) : f \in \calF\}.\]
    Then, 
    \[\frac{1}{2}\max_{\lossPred \in \calF} \adv(\lossPred) \leq \MCE(\calC, p) \leq \sqrt{\max_{\lossPred \in \calF'} \adv(\lossPred)}.\]
\end{theorem}

The proof of Theorem~\ref{thm:mc-conv-vs-loss-pred} can be found in Appendix~\ref{sec:mc-conv-cs-loss-pred-pf} and follows from two key lemmas. Lemma~\ref{lem:adv-implies-mc-err} establishes the left-hand inequality by showing how a loss predictor with good advantage can be used to construct a witness of large multicalibration error. Conversely, Lemma~\ref{lem:mc-err-implies-adv} establishes the right-hand inequality by showing how to leverage a witness for large multicalibration error to construct a loss predictor with large advantage. 

Before presenting our main lemmas, we introduce two auxiliary claims that are well-known in the literature on boosting and gradient descent. We provide proofs here for completeness and notational consistency.

Let $\mD'$ be a distribution over $(x, z) \in X \times [0,1]$. Let $h_1, h_2: \X \to [0,1]$ be two hypotheses. Under what conditions does $h_2$ improve on $h_1$? The following lemma gives a necessary condition: the update $\delta(x) = h_2(x) - h_1(x)$ must be correlated with the residual errors $z - h_1(x)$ of the hypothesis $h_1$ under the distribution $\mD'$.

\begin{claim}
\label{lem:improve-1}
    For two hypotheses $h_1, h_2$, 
    \[ \E_{\mD'}[(h_1(x) - z)^2] - \E[(h_2(x) - z)^2] \leq 2\E[(h_2(x) - h_1(x))(z - h_1(x))]. \]
\end{claim}
\begin{proof}
    Let us write $\delta(x) = h_2(x) - h_1(x)$. Then we have
    \begin{align*}
        \E[(h_1(x) - z)^2] - \E[(h_2(x) - z)^2] &= \E_{\mD'}[(h_1(x) -z)^2 - (h_1(x)  -z + \delta(x) )^2]\\
        &= -2\E[(h_1(x) - z)\delta(x)] - \E[\delta(x)^2]\\
        & \leq 2\E_{\mD}[(z - h_1(x))\delta(x)].
    \end{align*}
\end{proof}

Conversely, if we find an update $\delta(x)$ which is correlated with the residuals, we can perform a gradient descent update to reduce the squared error. We let $\Pi_{[0,1]}:\R \to [0,1]$ denote the projection operator onto the unit interval.

\begin{claim}
\label{lem:improve-2}    
If there exists $\delta:\X \to [-1,1]$ such that $\E_{\mD'}[\delta(x)(z - h_1(x))] \geq \beta \geq 0$, then setting $h_2(x) = \Pi_{[0,1]}(h_1(x)+ \beta \delta(x))$ gives
    \[ \E_{\mD'}[(h_1(x) - z)^2] - \E[(h_2(x) - z)^2] \geq \beta^2.\]
\end{claim}

\begin{proof}
    Without projection, we can write the gap in squared error as 
    \begin{align*}
        \E_{\mD'}[(h_1(x) - z)^2] - \E[(h_1(x) + \beta\delta(x) - z)^2] &=  2\beta \E_{\mD'}[(z - h_1(x))\delta(x)] - \beta^2\E[\delta(x)^2]\\
        &\geq 2\beta^2 - \beta^2 = \beta^2.
    \end{align*}
    While $h_1(x) + \beta \delta(x)$ may not be bounded in $[0,1]$, projection onto the interval can only further reduce the squared error.
\end{proof}

With these in hand, we show that any loss predictor with a non-trivial advantage points us to a failure of multicalibration. 

\begin{lemma}\label{lem:adv-implies-mc-err}
    Assume that $\lossPred$ achieves advantage $\alpha \geq 0$ over the self-entropy predictor. Then the function $\delta(\phi(p, x)) = \lossPred(\phi(p, x)) - \clp(p(x))$ satisfies
    \[ \E[\delta(\phi(p, x))H_{\ell}'(p(x))(y - p(x))] \geq \alpha/2.\]
    In other words, $\lossPred$ can be used to construct a witness $c(\phi(p, x)) = \delta(\phi(p, x))H_{\ell}'(p(x))$ for a multicalibration violation of magnitude $\alpha/2$. 
\end{lemma}
\begin{proof}
    Consider the loss regression problem, where we draw $(x,y) \in \X \times \zo \sim \mD$ and then return the pair $(x, z = \ell(y, p(x))$. We will use Claim \ref{lem:improve-1}, where we take $h_1 = \clp$ to be the self-entropy predictor and $h_2 = \lossPred$. 
    We can estimate the residual error of the self-entropy predictor as
    \begin{align}
    \label{eq:rewrite}
        \ell(y, p(x)) - \clp(p(x)) &= H_{\ell}(p(x)) + (y - p(x))H_{\ell}'(p(x)) - H_{\ell}(p(x))\notag\\
        &= (y - p(x))H_{\ell}'(p(x)).
    \end{align}
    By Claim \ref{lem:improve-1}, we have
    \begin{align*}
        \alpha &= \ex[(\ell(y, p(x)) - \clp(p(x)))^2] - \ex[(\ell(y, p(x)) - \lossPred(\phi(p, x))^2]\\
        &\leq 2\E[(\lossPred(\phi(p, x)) - \clp(p(x)))(\ell(y, p(x)) - \clp(p(x)))\\ &= 2\E[\delta(\phi(p, x))H_{\ell}'(p(x))(y - p(x))]
    \end{align*}
\end{proof}

Conversely to the result of Lemma~\ref{lem:adv-implies-mc-err}, we show that we can leverage certain types of multicalibration failures to predict loss with an advantage over the self-entropy predictor.

\begin{lemma}\label{lem:mc-err-implies-adv}
    Assume there exists a function $\delta:\Phi \to [-1,1]$ such that
    \[ \E[\delta(\phi(p, x))H_{\ell}'(p(x))(y - p(x))] \geq \beta \geq 0.\]
    i.e., the function $c(\phi(p, x)) = \delta(\phi(p, x))H_{\ell}'(p(x))$ is a witness for a multicalibration violation of magnitude $\beta$. Define the loss predictor
    \[ \lossPred(\phi(p, x)) = \Pi_{[0,1]}(\clp(p(x)) + \beta 
    \delta(\phi(p, x))).\] 
    Then $\adv(\lossPred) \geq \beta^2$. 
\end{lemma}
\begin{proof}
    We again consider the loss regression problem, 
    We now apply Lemma \ref{lem:improve-2} with $z = \ell(y, p(x))$, $h_1 = \clp$. The correlation condition we require is
    \[  \E[\delta(\phi(p, x))(\ell(y, p(x) - \clp(p(x)))] \geq \beta, \]
    By Equation \eqref{eq:rewrite}, we have
    \[ \E[\delta(\phi(p, x))(\ell(y, p(x)) - \clp(p(x)))] = \E[\delta(\phi(p, x))H_{\ell}'(p(x))(y - p(x))] \]
    which is at least $\beta$ by our assumption. Hence Claim \ref{lem:improve-2} implies that $h_2 = \lossPred$ has advantage $\beta^2$ over $\clp$.     
\end{proof}

\section{Loss prediction for multiple losses}
\label{sec:multiple-loss}

Up to this point, our discussion has focused on loss prediction for a single, predetermined loss function. However, in real-world applications, multiple stakeholders may use a predictor, each with unique objectives and priorities that correspond to different loss functions. This scenario would require training separate loss predictors for each user to meet their individual needs.

The self-entropy predictor offers a key advantage: it can be computed for any loss function using only the predictions $p(x)$, eliminating the need for additional training. Moreover, by extending the result of Theorem~\ref{thm:mc-conv-vs-loss-pred}, we can define a class test functions $\calC$ such that when $p$ is multicalibrated with respect to $\calC$, its self-entropy predictions simultaneously compete with the best-in-class loss predictor for each loss in a rich class of losses $\calL$, rather than just a fixed loss. We formalize this in the following lemma, which we prove in Appendix~\ref{sec:many-losses-mc-pf}:

\begin{lemma}\label{lem:many-losses-mc}
    Let $\calF$ be a class of loss predictors $f: \Phi \rightarrow [0, 1]$. Let $\calL$ be a class of bounded proper losses $\ell: \{0, 1\} \times [0, 1] \rightarrow [0, 1]$ with associated concave entropy functions $H_{\ell}: [0, 1] \rightarrow [0,1]$, and let $\calC_{\calL}$ be the class of test functions 
    \[\calC_{\calL} = \{(f(\phi(p, x)) - H_{\ell}(p(x)))H'_{\ell}(p(x)) : f \in \calF, \ell \in \calL\}.\]
    Then,
    \[\max_{\ell \in \calL}\max_{\lossPred \in \calF} \adv(\lossPred) \leq 2\MCE(\calC_{\calL}, p).\]
    I.e., no loss predictor from $\calF$ for any loss $\ell \in \calL$ can obtain better advantage than $2\MCE(\calC_{\calL}, p)$ over the self-entropy predictor. 
\end{lemma}

When $\calL$ is the set of all proper losses, the form of multicalibration imposed by $\calC_{\calL}$ can be thought of as the extension to multicalibration of the notion of \emph{proper calibration}, recently proposed by~\cite{OKK25}. The proper calibration error of a predictor $p$ is defined as 

\[\text{PCE}(p) = \max_{\ell \in \calL_{\text{prop}}}\left|\ex[H'_{\ell}(p(x))(y - p(x))]\right|\]
where $\calL_{\text{prop}}$ denotes the set of proper losses. Our condition can be thought of as ``proper multicalibration'' where each test function consists of $H_{\ell}'(p(x))$ multiplied with an additional test function $\delta(\phi(p, x))$, that may depend on other features in addition to the prediction value. 

\subsection{Achieving efficient multicalibration for many losses}

As the class of losses we consider expands, training an effective loss predictor for each individual loss becomes increasingly challenging. This section demonstrates that in certain scenarios, it is possible to efficiently produce a multicalibrated predictor with respect to the class of tests outlined in Lemma~\ref{lem:many-losses-mc}, even for some infinite classes of losses. This approach allows us to learn a single predictor $p$ whose self-entropy estimates can compete with the best $\lossPred \in \calF$ for every 
$\ell \in \calL$, thus eliminating the need to train separate predictors for each loss.

This result relies on the existence of a ``finite approximate basis'' (Definition~\ref{def:finite-approx-basis}) for the class of functions $\{H_{\ell}'\}_{\ell \in \calL}$, and is inspired by the techniques of~\cite{OKK25}, who use a similar approach to show the efficiency of proper calibration when $\{H_{\ell}'\}_{\ell \in \calL}$ has a finite approximate basis. 

We show a general version of this result in Theorem~\ref{thm:general-approx-basis-mc}, and instantiate it here for the class of 1-Lipschitz proper losses, $\calL_{Lip}$. 

The instantiation relies on a result proved by \cite{OKK25}, who show that $\{H_{\ell}'\}_{\ell \in \calL_{Lip}}$ has such a finite basis. We show efficiency in terms of oracle access to a weak-agnostic-learner for $\calF$, the class of loss predictors. We motivate this assumption by observing that if we care about learning a loss predictor from the class $\calF$, it's reasonable to assume that we have access to a weak agnostic learner for $\calF$. We formally define a weak agnostic learner as follows. 

\begin{definition}[Weak agnostic learner]\label{def:weak-agnostic-learner}
    Let $\alpha \geq 0$, $\delta \geq 0$. An $\alpha$-weak agnostic learner for $\calF \subseteq \{f: \Phi \rightarrow [-1, 1]\}$, closed under negation, with sample complexity $n$ and failure parameter $\delta$ is an algorithm that when given $n$ samples from a distribution $\calU$ over $\Phi \times [-1, 1]$, outputs $f \in \calF \cup \{\bot\}$ such that with probability at least $1 - \delta$ over the samples from $\calU$ and the randomness in the algorithm itself, if $\max_{f \in \calF} \ex_{(\phi, z) \sim \calU}[f(\phi)z] \geq \alpha,$
    the algorithm returns a $f \in \calF$ such that 
    $\ex_{(\phi, z) \sim \calU}[f(\phi)z] \geq \alpha/2.$
    Otherwise, if for all $f \in \calF$, 
    $\ex_{(\phi, z) \sim \calU}[f(\phi)z] \leq \alpha,$ the algorithm either returns $f = \bot$ or $f \in \calF$ such that $\ex_{(\phi, z) \sim \calU}[f(\phi)z] \geq \alpha/2.$
\end{definition}

With this definition in hand, we are ready to present the main theorem of this section. The proof can be found in Appendix~\ref{sec:1-lip-mc-pred-pf}.

\begin{theorem}\label{thm:1-lip-mc-pred}
    Fix $\delta, \epsilon > 0$. Let $\calL_{Lip}$ be the class of proper 1-Lipschitz losses $\ell:\{0, 1\} \times [0, 1] \rightarrow [0, 1]$, and let $\calF$ be a class of loss predictors $\calF: \Phi \rightarrow [-1, 1]$ that is closed under negation and contains the class of self entropy predictors, $\calH_{\calL_{Lip}} = \{H_{\ell}\}_{\ell \in \calL_{Lip}}$. Further assume that we have access to an $\alpha$-weak-agnostic-learner for $\calF$ with sample complexity $n$ and failure parameter $\beta \leq \frac{\alpha^2\delta}{4\lceil 2/\epsilon + 1\rceil}$. 

    Then, there exists an algorithm that, given $m = O(n/\alpha^2)$ samples, with probability at least $1 - \delta$ outputs a predictor $p$ such that 
    \[\max_{\ell \in \calL_{Lip}} \max_{\lossPred \in \calF} \adv(\lossPred) \leq 16\alpha + 4\epsilon.\]
\end{theorem}

In other words, our learned $p$'s self-entropy predictions compete with the best-in-class loss predictor with \emph{every} $\ell \in \calL_{Lip}$, up to an error of $16\alpha + 4\epsilon$.

\section{Experiments}
\label{sec:exp}

We have shown in Section \ref{sec:lp-mc} a correspondence between the advantage a loss predictor has over the self entropy predictor and the multicalibration error. In this section, we empirically demonstrate this correspondence and see that it holds across several base models, loss prediction algorithms, as well as data subgroups.

\paragraph{Experiment design.} We follow the basic design set forth in \cite{hansen2024mcp} for working with binary prediction tasks on UCI tabular datasets, specifically Credit Default \cite{default_of_credit_card_clients_350} and Bank Marketing \cite{bank_marketing_222}. For each dataset, we consider certain subgroups (13 and 15 different groups respectively) defined by combinations of features such as occupation, education, and gender (see Appendix C.4 \cite{hansen2024mcp} for full details). These subgroups are used to evaluate the multicalibration error of the predictors as follows: for each subgroup we measure the \emph{smoothed Expected Calibration Error} (smECE) \cite{blasiok2023smece}, and take the multicalibration error to be the maximum smECE obtained.

We examine base predictors from different model families at various levels of multicalibration, specifically Naive Bayes and Support Vector Machines (SVMs), which tend to be uncalibrated without any postprocessing, along with Random Forests, Logistic Regression, Decision Trees, and Multilayer Perceptrons (MLPs), which tend to be well-calibrated out of the box when trained with empirical risk minimization. The base predictor MLP we use has a three-hidden-layer architecture with ReLU activations. For further details on hyperparameters, architectures, and training, we point the reader to Appendix E in \cite{hansen2024mcp}. 

We then run the following four loss prediction algorithms: decision tree regression, XGBoost, support vector regression (SVR), and a three-hidden-layer MLP. Each of these is input aware, that is, it is given both $\inp(x)$ and $p(x)$ at train time, and is trained using a standard regression objective to minimize $\ex_{(x, y) \sim \calD}[\big(\lossPred(\inp(x), p(x)) - \ell(y, p(x)) \big)^2]$. Our target loss $\ell$ will be the squared loss $\ell(y, p(x)) = (y-p(x))^2$. 

\paragraph{Results.} Our main takeaways are as follows:
\begin{itemize}
    \item Loss predictors perform better as the multicalibration error of the base model increases.

    \item Loss predictors perform better on subgroups that exhibit higher calibration error. 

\end{itemize}
 
The first takeaway is demonstrated by Figure \ref{fig:cd:agg_adv}, where the horizontal axis indicates the max smECE of the base model, and the vertical axis indicates the advantage the loss predictor attains over the self-entropy predictor of the base model. 
As our theory predicts, we see a clear positive correlation between the maximum smECE and the relative performance of the loss predictor. In other words, less multicalibrated models have better performing loss predictors. This correlation holds across different base models and different algorithms for the loss predictor.

\begin{figure}[H] 
    \centering
    \includegraphics[width=16cm]{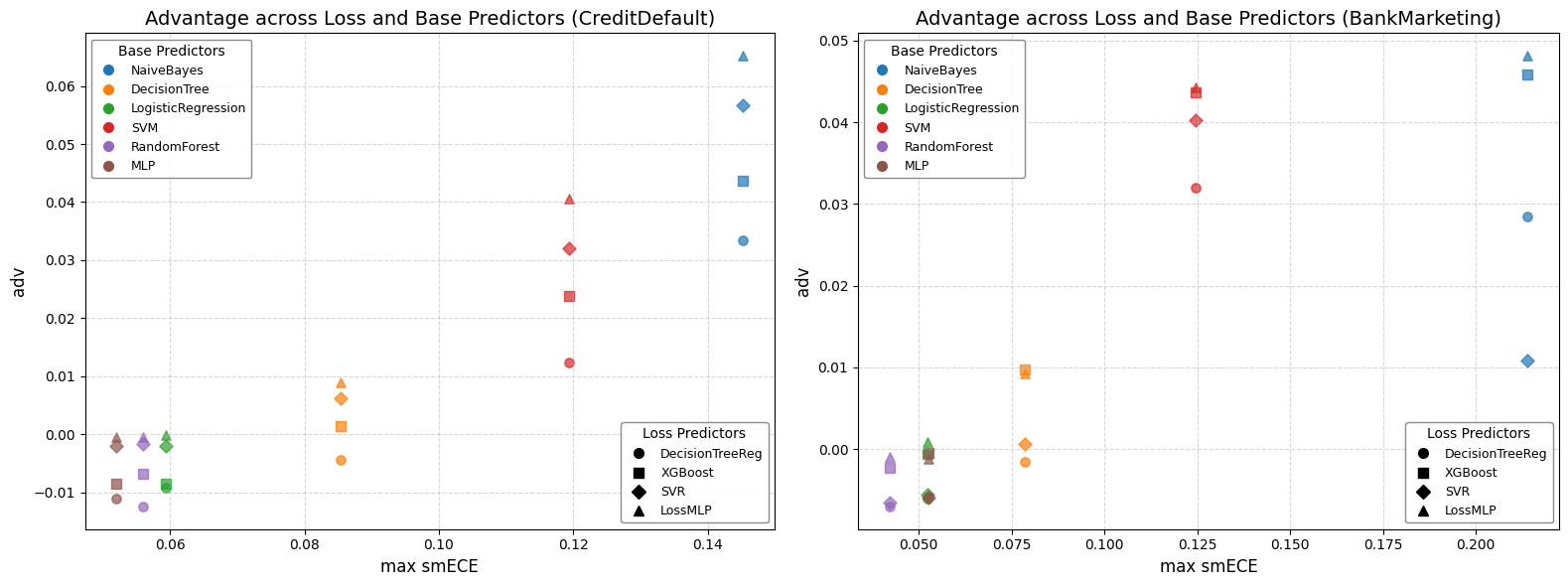}
    \caption{\small Advantage vs.\ max smECE across base predictors and loss prediction algorithms. For any particular loss predictor (shape), we see that as the multicalibration error of the base model (color) grows, the loss predictor's advantage improves.}
    \label{fig:cd:agg_adv}
\end{figure}

To delve deeper, we examine how loss prediction advantage varies across different subgroups. For this experiment we vary the base predictor only, while fixing the loss prediction algorithm to be an MLP. In Figure \ref{fig:cd:group_adv}, for each base predictor and each subgroup, we report the loss predictor's advantage restricted to the subgroup on the vertical axis and the smECE of the subgroup on the horizontal axis. 

For base models that are poorly calibrated overall (Decision Tree, SVM, and Naive Bayes), we see a clear correlation showing the loss predictor performs better on subgroups as the calibration error gets worse. By contrast, base predictors that are well-calibrated overall (Logistic Regression, Random Forest, and MLPs) allow negligible loss prediction advantage even after stratifying by subgroup.

\begin{figure}[H] 
    \centering
    \includegraphics[width=16cm]{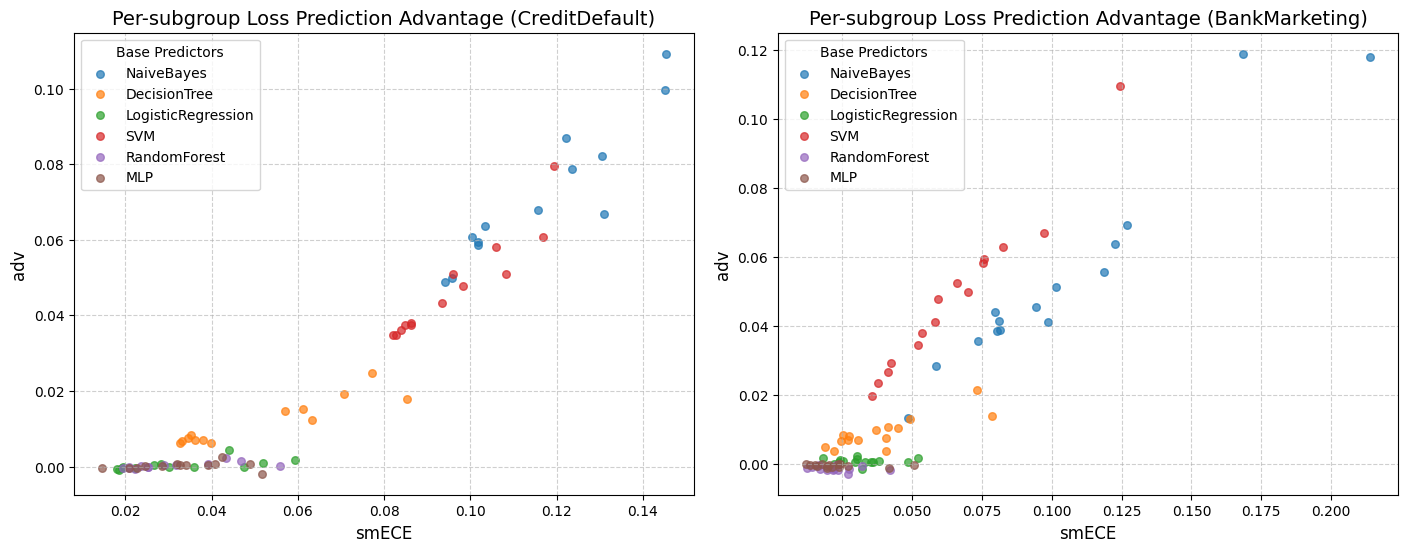}
    \caption{\small Fixing the type of loss predictor to be an MLP, we plot the loss advantage vs.\ the smECE on each subgroup across different base predictor models for the Credit Default and Bank Marketing datasets. For a fixed base predictor (color), the loss predictor exhibits more advantage on subgroups where the base predictor is less calibrated.}
    \label{fig:cd:group_adv}
\end{figure}

\section{Related work}
\label{sec:related}

\paragraph{Applications of loss prediction.} The idea of loss prediction has its roots in Bayesian and decision-theoretic active learning  \cite{settles2009active,ren2021survey}, wherein the loss expected to be incurred at a point provides a natural measure of how valuable it is to label; see e.g.\ the well-known method of Expected Error Reduction (EER) \cite{roy2001toward}. To our knowledge, loss prediction in the explicit sense that we consider in this paper was first studied by \cite{yoo2019learning}, who proposed training an auxiliary loss prediction module alongside the base predictor. This is a practical approach used in many real-world applications. An important example from industry is the popular Segment Anything Model \cite{kirillov2023segment}, which is an image segmentation model that includes an IoU prediction module\footnote{IoU, or intersection over union, is a standard segmentation error metric.}. This module plays a key role in the continual learning ``data engine'' used to train the model. Other applications of loss prediction are in routing inputs to weak or strong models \cite{dinghybrid, ong2024routellm, hu2024routerbench}, diagnosing model failures \cite{jain2022distilling}, and MRI reconstruction \cite{hu2021learning}.

Loss prediction is inherently connected to the broader topic of uncertainty quantification \cite{hullermeier2021aleatoric,abdar2021review}. The work of \cite{lahlou2021deup} formulates epistemic uncertainty as a form of excess loss or risk (see also \cite{xu2022minimum}) and estimates it using an auxiliary loss predictor. Loss decomposition and connections to calibration have also been studied by \cite{kull2015novel,ahdritz2024provable}, although these works do not discuss auxiliary loss predictors per se.

\paragraph{Related theoretical work.}  We are not aware of prior theoretical work that studies the complexity of loss prediction. We build on notions and techniques from prior work on calibration \cite{decisionCal, gopalan2022low, kleinberg2023u}, multicalibration \cite{hebert2018multicalibration, kim2022universal}, omniprediction \cite{omni, lossOI, OP3, OKK25} and outcome indistinguishability \cite{OI, OI2}. Multicalibration has found applications to a myriad areas beyond multigroup fairness; a partial list includes omniprediction \cite{omni}, domain adaptation \cite{kim2022universal}, pseudorandomness \cite{OI, DworkLLT23} and computational complexity \cite{CasacubertaDV24}. Our work adds loss prediction to this list.  

\paragraph{Decision calibration, decision OI and proper calibration.}
The work of \cite{decisionCal} on decision calibration  (implicitly) considered the self-entropy predictor, and
conditions under which this predictor is accurate in expectation. The subsequent work of \cite{lossOI} termed this condition decision OI and showed that calibration  of the predictor guarantees that the self-entropy predictor is itself calibrated for loss prediction. 
As we have seen, however, calibration of the predictor is not necessary for the self-entropy predictor to be calibrated (or optimal), due to the existence of blind-spots for a specific loss. This is explained by the notion of proper calibration introduced in
\cite{OKK25}, who showed that it tightly characterizes decision OI.

While our results have strong connections to all these works, the key difference is that our goal in loss prediction is not just to give loss estimates that are calibrated, it is to predict the true loss in the regression sense, which is potentially a much harder task.

 \paragraph{Swap multicalibration.} Our results equate the ability to gain an advantage over the self-predictor to a lack of multicalibration. This recalls a result of \cite{OP3} which characterizes swap omniprediction by multicalibration: there too, the ability to achieve better loss than a simple baseline is attributable to a lack of multicalibration. Another related work is that of \cite{Globus-HarrisHK23}, which views multicalibration as a boosting algorithm for regression. Their work also connects multicalibration and regression, but  the regression task they analyze is predicting $y$, whereas the regression task that we analyze is predicting $\ell(y, p(x))$.

\paragraph{Representation-aware multicalibration.} 
Internal representation-aware multiaccuracy is considered in the work of \cite{kgz} who use it in the context of face-recognition. Internal representation-aware multicalibration  has connections to the notion of Code-Access Outcome Indistinguishability (OI), proposed by~\cite{dwork2021outcome}. In Code-Access OI, outcomes generated by $p(x)$ must be indistinguishable from the true outcomes generated from the target distribution with respect to a set of tests that can inspect the full definition and code of $p$. With code access, such tests can compute the internal features $r_p(x)$ that are available in internal representation-aware multicalibration, but also have other capabilities such as querying $p$ on other points $x' \in \calX$. The use of external representations for auditing/improving predictions is found in the work of \cite{BuolamwiniG18}
who use skin type to assess facial recognition; in recent work of $\cite{alur24}$, which investigates the use of expert opinions in addition to ML predictions to improve on medical test results, and in the work of \cite{jain2022distilling} who use representations from a foundation model that is distinct from the model they audit.  
\paragraph{Experimental work on multicalibration.} The work of \cite{kgz} considers internal-representation aware multiaccuracy for facial recognition tasks, and shows how auditing can be used to improve accuracy across subgroups. 
Recently, \cite{hansen2024mcp}  conducted a systematic investigation of multicalibration in practice, analyzing the utility of algorithms for multicalibration on a number of real-world datasets, prediction models, and demographic subgroups. We build closely on their setup for our own experiments.

\section*{Acknowledgements}

We thank Moises Goldszmidt, Shayne O'Brien, Daniel Tsai, Robert Fisher and Dor Shaviv for introducing us to this problem and for sharing their insights on and experience with practical loss prediction and its applications. AK is additionally
supported by the Paul and Daisy Soros Fellowship for New Americans. CP is supported by the Apple Scholars in AI/ML PhD fellowship.

\bibliographystyle{alpha}
\bibliography{ref}

\newcommand{\etalchar}[1]{$^{#1}$}
\begin{thebibliography}{GHK{\etalchar{+}}23b}

\bibitem[AGG{\etalchar{+}}25]{ahdritz2024provable}
Gustaf Ahdritz, Aravind Gollakota, Parikshit Gopalan, Charlotte Peale, and Udi Wieder.
\newblock Provable uncertainty decomposition via higher-order calibration.
\newblock In {\em International Conference on Learning Representations}, 2025.
\newblock To appear.

\bibitem[ALL{\etalchar{+}}24]{alur24}
Rohan Alur, Loren Laine, Darrick~K. Li, Dennis Shung, Manish Raghavan, and Devavrat Shah.
\newblock Integrating expert judgment and algorithmic decision making: An indistinguishability framework, 2024.

\bibitem[APH{\etalchar{+}}21]{abdar2021review}
Moloud Abdar, Farhad Pourpanah, Sadiq Hussain, Dana Rezazadegan, Li~Liu, Mohammad Ghavamzadeh, Paul Fieguth, Xiaochun Cao, Abbas Khosravi, U~Rajendra Acharya, et~al.
\newblock A review of uncertainty quantification in deep learning: Techniques, applications and challenges.
\newblock {\em Information fusion}, 76:243--297, 2021.

\bibitem[BG18]{BuolamwiniG18}
Joy Buolamwini and Timnit Gebru.
\newblock Gender shades: Intersectional accuracy disparities in commercial gender classification.
\newblock In {\em Conference on Fairness, Accountability and Transparency, {FAT} 2018, 23-24 February 2018, New York, NY, {USA}}, volume~81 of {\em Proceedings of Machine Learning Research}, pages 77--91. {PMLR}, 2018.

\bibitem[CDV24]{CasacubertaDV24}
S{\'{\i}}lvia Casacuberta, Cynthia Dwork, and Salil~P. Vadhan.
\newblock Complexity-theoretic implications of multicalibration.
\newblock In {\em 56th Annual {ACM} Symposium on Theory of Computing, {STOC} 2024}, pages 1071--1082. {ACM}, 2024.

\bibitem[DKR{\etalchar{+}}21a]{OI}
Cynthia Dwork, Michael~P. Kim, Omer Reingold, Guy~N. Rothblum, and Gal Yona.
\newblock Outcome indistinguishability.
\newblock In {\em ACM Symposium on Theory of Computing (STOC'21)}, 2021.

\bibitem[DKR{\etalchar{+}}21b]{dwork2021outcome}
Cynthia Dwork, Michael~P Kim, Omer Reingold, Guy~N Rothblum, and Gal Yona.
\newblock Outcome indistinguishability.
\newblock In {\em Proceedings of the 53rd Annual ACM SIGACT Symposium on Theory of Computing}, pages 1095--1108, 2021.

\bibitem[DKR{\etalchar{+}}22]{OI2}
Cynthia Dwork, Michael~P. Kim, Omer Reingold, Guy~N. Rothblum, and Gal Yona.
\newblock Beyond bernoulli: Generating random outcomes that cannot be distinguished from nature.
\newblock In {\em The 33rd International Conference on Algorithmic Learning Theory}, 2022.

\bibitem[DLLT23]{DworkLLT23}
Cynthia Dwork, Daniel Lee, Huijia Lin, and Pranay Tankala.
\newblock From pseudorandomness to multi-group fairness and back.
\newblock In {\em 36th Annual Conference on Learning Theory, {COLT} 2023}, volume 195 of {\em Proceedings of Machine Learning Research}, pages 3566--3614. {PMLR}, 2023.

\bibitem[DMW{\etalchar{+}}24]{dinghybrid}
Dujian Ding, Ankur Mallick, Chi Wang, Robert Sim, Subhabrata Mukherjee, Victor R{\"u}hle, Laks~VS Lakshmanan, and Ahmed~Hassan Awadallah.
\newblock Hybrid llm: Cost-efficient and quality-aware query routing.
\newblock In {\em The Twelfth International Conference on Learning Representations}, 2024.

\bibitem[FV99]{FV99}
Dean Foster and Rakesh Vohra.
\newblock Regret in the on-line decision problem.
\newblock {\em Games and Economic Behavior}, 29:7--35, 1999.

\bibitem[GHK{\etalchar{+}}23a]{Globus-HarrisHK23}
Ira Globus{-}Harris, Declan Harrison, Michael Kearns, Aaron Roth, and Jessica Sorrell.
\newblock Multicalibration as boosting for regression.
\newblock In {\em International Conference on Machine Learning, {ICML} 2023}, 2023.

\bibitem[GHK{\etalchar{+}}23b]{lossOI}
Parikshit Gopalan, Lunjia Hu, Michael~P. Kim, Omer Reingold, and Udi Wieder.
\newblock Loss minimization through the lens of outcome indistinguishability.
\newblock In {\em Innovations in theoretical computer science (ITCS'23)}, 2023.

\bibitem[GKR{\etalchar{+}}22]{omni}
Parikshit Gopalan, Adam~Tauman Kalai, Omer Reingold, Vatsal Sharan, and Udi Wieder.
\newblock Omnipredictors.
\newblock In {\em Innovations in Theoretical Computer Science (ITCS'2022)}, 2022.

\bibitem[GKR23]{OP3}
Parikshit Gopalan, Michael~P. Kim, and Omer Reingold.
\newblock Swap agnostic learning, or characterizing omniprediction via multicalibration.
\newblock In {\em NeurIPS'23}, 2023.

\bibitem[GKSZ22]{gopalan2022low}
Parikshit Gopalan, Michael~P Kim, Mihir~A Singhal, and Shengjia Zhao.
\newblock Low-degree multicalibration.
\newblock In {\em Conference on Learning Theory}, pages 3193--3234. PMLR, 2022.

\bibitem[GPM{\etalchar{+}}14]{GANs}
Ian~J. Goodfellow, Jean Pouget{-}Abadie, Mehdi Mirza, Bing Xu, David Warde{-}Farley, Sherjil Ozair, Aaron~C. Courville, and Yoshua Bengio.
\newblock Generative adversarial nets.
\newblock In {\em NeurIPS'14}, pages 2672--2680, 2014.

\bibitem[GR07]{gneiting2007strictly}
Tilmann Gneiting and Adrian~E Raftery.
\newblock Strictly proper scoring rules, prediction, and estimation.
\newblock {\em Journal of the American statistical Association}, 102(477):359--378, 2007.

\bibitem[HBL{\etalchar{+}}24]{hu2024routerbench}
Qitian~Jason Hu, Jacob Bieker, Xiuyu Li, Nan Jiang, Benjamin Keigwin, Gaurav Ranganath, Kurt Keutzer, and Shriyash~Kaustubh Upadhyay.
\newblock Routerbench: A benchmark for multi-llm routing system.
\newblock {\em arXiv preprint arXiv:2403.12031}, 2024.

\bibitem[HDNS24]{hansen2024mcp}
Dutch Hansen, Siddartha Devic, Preetum Nakkiran, and Vatsal Sharan.
\newblock {When is Multicalibration Post-Processing Necessary?}
\newblock In {\em Advances in Neural Information Processing Systems}, 2024.

\bibitem[HKRR18]{hebert2018multicalibration}
Ursula {H{\'e}bert-Johnson}, Michael Kim, Omer Reingold, and Guy Rothblum.
\newblock Multicalibration: Calibration for the (computationally-identifiable) masses.
\newblock In {\em International Conference on Machine Learning}, pages 1939--1948. PMLR, 2018.

\bibitem[HPW21]{hu2021learning}
Shi Hu, Nicola Pezzotti, and Max Welling.
\newblock Learning to predict error for mri reconstruction.
\newblock In {\em Medical Image Computing and Computer Assisted Intervention--MICCAI 2021}, pages 604--613. Springer, 2021.

\bibitem[HW21]{hullermeier2021aleatoric}
Eyke H{\"u}llermeier and Willem Waegeman.
\newblock Aleatoric and epistemic uncertainty in machine learning: An introduction to concepts and methods.
\newblock {\em Machine learning}, 110(3):457--506, 2021.

\bibitem[JLMM23]{jain2022distilling}
Saachi Jain, Hannah Lawrence, Ankur Moitra, and Aleksander Madry.
\newblock Distilling model failures as directions in latent space.
\newblock In {\em International Conference on Learning Representations}, 2023.

\bibitem[JN24]{blasiok2023smece}
Błasiok Jarosław and Preetum Nakkiran.
\newblock {Smooth ECE: Principled Reliability Diagrams via Kernel Smoothing}.
\newblock In {\em International Conference on Learning Representations}, 2024.

\bibitem[KF15]{kull2015novel}
Meelis Kull and Peter Flach.
\newblock Novel decompositions of proper scoring rules for classification: Score adjustment as precursor to calibration.
\newblock In {\em Machine Learning and Knowledge Discovery in Databases: European Conference, ECML PKDD 2015}, pages 68--85. Springer, 2015.

\bibitem[KGZ19]{kgz}
Michael~P. Kim, Amirata Ghorbani, and James Zou.
\newblock Multiaccuracy: Black-box post-processing for fairness in classification.
\newblock In {\em Proceedings of the 2019 AAAI/ACM Conference on AI, Ethics, and Society}, pages 247--254, 2019.

\bibitem[KKG{\etalchar{+}}22]{kim2022universal}
Michael~P Kim, Christoph Kern, Shafi Goldwasser, Frauke Kreuter, and Omer Reingold.
\newblock Universal adaptability: Target-independent inference that competes with propensity scoring.
\newblock {\em Proceedings of the National Academy of Sciences}, 119(4), 2022.

\bibitem[KLST23]{kleinberg2023u}
Bobby Kleinberg, Renato~Paes Leme, Jon Schneider, and Yifeng Teng.
\newblock U-calibration: Forecasting for an unknown agent.
\newblock In {\em The Thirty Sixth Annual Conference on Learning Theory}, pages 5143--5145. PMLR, 2023.

\bibitem[KMR{\etalchar{+}}23]{kirillov2023segment}
Alexander Kirillov, Eric Mintun, Nikhila Ravi, Hanzi Mao, Chloe Rolland, Laura Gustafson, Tete Xiao, Spencer Whitehead, Alexander~C Berg, Wan-Yen Lo, et~al.
\newblock Segment anything.
\newblock In {\em Proceedings of the IEEE/CVF International Conference on Computer Vision}, pages 4015--4026, 2023.

\bibitem[LJN{\etalchar{+}}21]{lahlou2021deup}
Salem Lahlou, Moksh Jain, Hadi Nekoei, Victor~Ion Butoi, Paul Bertin, Jarrid Rector-Brooks, Maksym Korablyov, and Yoshua Bengio.
\newblock {DEUP: Direct epistemic uncertainty prediction}.
\newblock {\em arXiv preprint arXiv:2102.08501}, 2021.

\bibitem[MRC14]{bank_marketing_222}
S.~Moro, P.~Rita, and P.~Cortez.
\newblock {Bank Marketing}.
\newblock UCI Machine Learning Repository, 2014.
\newblock {DOI}: https://doi.org/10.24432/C5K306.

\bibitem[OAW{\etalchar{+}}24]{ong2024routellm}
Isaac Ong, Amjad Almahairi, Vincent Wu, Wei-Lin Chiang, Tianhao Wu, Joseph~E Gonzalez, M~Waleed Kadous, and Ion Stoica.
\newblock Routellm: Learning to route llms with preference data.
\newblock {\em arXiv preprint arXiv:2406.18665}, 2024.

\bibitem[OKK25]{OKK25}
Princewill Okoroafor, Robert Kleinberg, and Michael~P Kim.
\newblock Near-optimal algorithms for omniprediction.
\newblock {\em arXiv preprint arXiv:2501.17205}, 2025.

\bibitem[RKH{\etalchar{+}}21]{radford2021learning}
Alec Radford, Jong~Wook Kim, Chris Hallacy, Aditya Ramesh, Gabriel Goh, Sandhini Agarwal, Girish Sastry, Amanda Askell, Pamela Mishkin, Jack Clark, et~al.
\newblock Learning transferable visual models from natural language supervision.
\newblock In {\em International conference on machine learning}, pages 8748--8763. PMLR, 2021.

\bibitem[RM01]{roy2001toward}
Nicholas Roy and Andrew McCallum.
\newblock Toward optimal active learning through monte carlo estimation of error reduction.
\newblock In {\em International Conference on Machine Learning}, 2001.

\bibitem[RXC{\etalchar{+}}21]{ren2021survey}
Pengzhen Ren, Yun Xiao, Xiaojun Chang, Po-Yao Huang, Zhihui Li, Brij~B Gupta, Xiaojiang Chen, and Xin Wang.
\newblock A survey of deep active learning.
\newblock {\em ACM computing surveys (CSUR)}, 54(9):1--40, 2021.

\bibitem[Set09]{settles2009active}
Burr Settles.
\newblock Active learning literature survey, 2009.

\bibitem[XR22]{xu2022minimum}
Aolin Xu and Maxim Raginsky.
\newblock Minimum excess risk in bayesian learning.
\newblock {\em IEEE Transactions on Information Theory}, 68(12):7935--7955, 2022.

\bibitem[Yeh09]{default_of_credit_card_clients_350}
I-Cheng Yeh.
\newblock {Default of Credit Card Clients}.
\newblock UCI Machine Learning Repository, 2009.
\newblock {DOI}: https://doi.org/10.24432/C55S3H.

\bibitem[YK19]{yoo2019learning}
Donggeun Yoo and In~So Kweon.
\newblock Learning loss for active learning.
\newblock In {\em Proceedings of the IEEE/CVF conference on computer vision and pattern recognition}, pages 93--102, 2019.

\bibitem[ZKS{\etalchar{+}}21]{decisionCal}
Shengjia Zhao, Michael~P. Kim, Roshni Sahoo, Tengyu Ma, and Stefano Ermon.
\newblock Calibrating predictions to decisions: A novel approach to multi-class calibration.
\newblock In {\em Advances in Neural Information Processing Systems}, 2021.

\end{thebibliography}
\appendix
\section{Handling non-proper losses}
\label{sec:non-proper}

We consider an abstract action space $\mA$; examples are the discrete setting where $\mA = [k]$, and the continuous setting where $\mA = \R$. 
A hypothesis is a function $h: \X \to \mA$. 
A loss function is a function $\ell: \zo \times \mA \to [0,1]$. The expected loss of hypothesis $h$ at the point $x$ is given by $\ex[\ell(y, h(x))|x]$. The goal of a loss predictor is to learn a function $\lossPred: \Phi \to \R$ that gives pointwise estimates of this quantity. As in definition \ref{def:lp}, we can define a hierarchy of loss predictors based on the features available to them.

For any loss $\ell$, if the labels are drawn according to $y \sim \Ber(p)$ for any $p \in [0, 1]$, then the optimal prediction that minimizes the loss, $k_{\ell}(p) \in [0, 1]$ is defined by 
\[k_{\ell}(p) = \argmin_{v \in [0, 1]}\ex_{y \sim \Ber(\pbayes(x))}[\ell(y, v)]
\footnote{In the event that there is no unique minimum, we allow $k_\ell$ to output a subset of $\mA$, so it is strictly speaking a relation rather than a function. We blur this distinction for simplicity}.\]

If there exist a {\em latent} predictor $p_h:\X \to [0,1]$ so that $h = k_\ell \circ p_h$ is obtained by best-responding to its predictions, then we can reduce to the setting of proper losses, since 
\[ \ex[\ell(y, h(x))] = \ex[\ell(y, k_\ell(p(x)) = \ex [\ell \circ k_\ell(y, p(x))] \]
and we have the following result of \cite{kleinberg2023u}.

\begin{lemma}[\cite{kleinberg2023u}]
\label{lem:klst}
    For any loss $\ell:\zo \times \mA \to [0,1]$, the loss function $\ell \circ k_\ell: \zo \times [0,1] \to \R$ is a proper loss.
\end{lemma}

But under what conditions on $h$ does there exist such a predictor $p_h$? And is it easy to estimate its predictions given access to $h$? 
 
To answer the first question, we show that it is equivalent to assuming that the hypothesis satisfies a simple optimality condition for the loss.

\begin{definition}
    The hypothesis $h:\X \to [0,1]$ is swap-optimal for $\mD$ if for every function $\kappa:\mA \to \mA$, it holds that $\ex[\ell(y, h(x))] \leq \ex[\ell(y, \kappa(h(x)))]$.
\end{definition}
Swap optimality is a weak optimality condition that can be easily achieved by post-processing. It is quite reasonable to assume that a well-trained model optimized to minimize loss satisfies this guarantee. For instance, a well-trained image classifier should not improve if every time it predicts {\em cat}, we say {\em dog} instead. 
For a swap optimal hypothesis $h$, we show that is indeed easy to identify a latent predictor $p_h$ so that $h$ is obtained by best-responding to its predictions. This theorem lets us extend our theory of loss prediction for proper losses to arbitrary loss functions under the rather weak assumption that $h$ is swap-optimal. 

\begin{theorem}
\label{thm:swap-proper}
Given a hypothesis $h:\X \to \mA$ and a distribution $\mD$, define the predictor $p_h: \X \to [0,1]$ by $p_h(x) = \ex_\mD[y|h(x)]$. The hypothesis $h$ is swap optimal iff $h(x)= k_\ell \circ p_h(x)$ for all $x \in \X$.\footnote{Strictly speaking, $k_\ell$ is not a function as there can be many optimal actions. However we interpret this equation as saying $h(x)$ is in the set of optimal actions for $p_h(x)$.}
\end{theorem}
\begin{proof}
    Assume that $h$ is not swap-optimal, so there exist $\kappa$ such that $\ex[\ell(y, \kappa(h(x)))] < \ex[\ell(y, h(x))]$. There must exist a specific choice of $h(x) = a \in \mA$ conditioned on which the inequality still holds, hence 
    \[ \ex[\ell(y, \kappa(a))|h(x) = a]  \leq \ex[\ell(y, a)|h(x) =a]. \]
    Let $\ex[y|h(x) =a] = v$. But this shows that when $y \sim \Ber(v)$, $\ex[\ell(y, \kappa(a)] < \ex[\ell(y,a)]$, so $a \neq k_\ell(v)$. Hence for all $x \in h^{-1}(a)$, we have $h(x) = a \neq k_\ell(v) = k_\ell(p_h(x))$. 

    Conversely, if $h$ is indeed swap optimal, then it must be the case that every action $a \in \mA$ is a best response to $\E[y|h(x) =a] = p_h(x)$, which means we have $h(x) = k_\ell(p_h(x))$.
\end{proof}

\section{Proofs from Section~\ref{sec:lp-mc}}

\subsection{Proof of Theorem~\ref{thm:mc-conv-vs-loss-pred}}\label{sec:mc-conv-cs-loss-pred-pf}

\begin{proof}[Proof of Theorem~\ref{thm:mc-conv-vs-loss-pred}]
    The inequality on the left follows from Theorem~\ref{lem:adv-implies-mc-err}, while the inequality on the right follows from Lemma~\ref{lem:mc-err-implies-adv}. We prove each in turn, starting with the left-hand inequality. 

    By Theorem~\ref{lem:adv-implies-mc-err}, if there exists a $f \in \calF$ such that setting $\lossPred = f$ gives $\adv(\lossPred) = \alpha$, then this implies that 
    \[\ex[(\lossPred(\phi(p, x) - \clp(p(x)))H_{\ell}'(p(x))(y - p(x))] \geq \alpha/2.\]
    We observe that because $\lossPred = f \in \calF$, the witness of this multicalibration violation, $(\lossPred(\phi(p, x) - \clp(p(x)))H_{\ell}'(p(x))$ lies in $\calC$, and thus 
    \begin{align*}
        \MCE(\calC, p) &= \max_{c \in \calC}\left| \ex[c(\phi(p, x))(y - p(x))]\right|\\
        & \geq \ex[(\lossPred(\phi(p, x) - \clp(p(x)))H_{\ell}'(p(x))(y - p(x))]\\
        &\geq \alpha/2\\
        &= \frac{1}{2}\adv(\lossPred)
    \end{align*}

    The inequality follows by taking the maximum over all $\lossPred \in \calF$, as $\lossPred$ was chosen arbitrarily. 

    We now move on to proving the inequality on the right, i.e., the upper bound on $\MCE(\calC, p)$. 

    By definition of the multicalibration error and $\calC$, there exists some $c \in \calC$ that witnesses a multicalibration error of magnitude $\MCE(\calC, p)$, i.e. for some $f \in \calF$, 

    \[\MCE(\calC, p) = \left|\underbrace{\ex[(f(\phi(p, x)) - H_{\ell}(p(x)))H_{\ell}'(p(x))(y - p(x))]}_{:= E_f}\right|. \]

    Thus, if we define $\delta: \Phi \rightarrow [-1, 1]$ as 
    \[\delta(\phi(p, x)) = \sgn(E_f)(f(\phi(p, x)) - H_{\ell}(p(x))),\]
    it follows that 
    \begin{align*}
        \ex[\delta(\phi(p, x))H_{\ell}'(p(x))(y - p(x))] &= \MCE(\calC, p).
    \end{align*}

    Applying Lemma~\ref{lem:mc-err-implies-adv} for this $\delta$ implies that for the loss predictor defined by $\lossPred(\phi(p, x)) = \Pi_{[0,1]}(\clp(p(x)) + \MCE(\calC, p) 
    \delta(\phi(p, x)))$ satisfies 

    \[\adv(\lossPred(\phi(p, x))) \geq \MCE(\calC, p)^2.\]

    The proof of the inequality follows by observing that $\lossPred \in \calF'$, because 
    \begin{align*}
        \lossPred(\phi(p, x)) &= \Pi_{[0,1]}(\clp(p(x)) + \MCE(\calC, p) 
    \delta(\phi(p, x)))\\
    &= \Pi_{[0,1]}(H_{\ell}(p(x)) + \underbrace{\MCE(\calC, p) 
    \sgn(E_f)}_{:= \beta}(f(\phi(p, x)) - H_{\ell}(p(x))))\\
    &= \Pi_{[0,1]}((1 - \beta)H_{\ell}(p(x)) + \beta f(\phi(p, x)))
    \end{align*}

    Where $\beta = \sgn(E_f)\MCE(\calC, p) \in [-1, 1]$, because $\MCE(\calC, p) \in [0, 1]$. 

    Thus, $\lossPred \in \calF'$, and so we conclude that 

    \[\max_{\lossPred \in \calF'}\adv(\lossPred(\phi(p, x))) \geq \MCE(\calC, p)^2.\]

    We get the right-hand inequality from the statement after taking square root of both sides. 
\end{proof}

\section{Extended discussion and proofs from Section~\ref{sec:multiple-loss}}

\subsection{Proof of Lemma~\ref{lem:many-losses-mc}}\label{sec:many-losses-mc-pf}
\begin{proof}[Proof of Lemma~\ref{lem:many-losses-mc}]
    For a fixed loss $\ell \in \calL$, let 
    \[\calC_{\ell} = \{(f(\phi(p, x)) - H_{\ell}(p(x)))H'_{\ell}(p(x)) : f \in \calF\}.\]

    By Theorem~\ref{thm:mc-conv-vs-loss-pred}, we can guarantee that 
    \[\max_{\lossPred \in \calF} \adv(\lossPred) \leq 2\MCE(\calC_{\ell}, p).\]
    Taking the max over $\calL$ for both sides, we get
    \[\max_{\ell \in \calL}\max_{\lossPred \in \calF} \adv(\lossPred) \leq \max_{\ell \in \calL} 2\MCE(\calC_{\ell}, p) \leq 2\MCE(\calC_{\calL}, p).\]

    Where the right-most inequality follows from the fact that $\calC_{\calL} = \bigcup_{\ell \in \calL} \calC_{\ell}$. This proves the desired inequality.
\end{proof}

\subsection{Multicalibration for product classes}

In this section, we introduce some useful notation that we will use to refer to and relate certain classes of multicalibration test functions. 

\begin{definition}\label{def:prod-class}
    Let $\calA \subseteq \{a: \Phi \rightarrow [-1, 1]\}$ and $\calB \subseteq \{b: \Phi \rightarrow [-1, 1]\}$ be two classes of functions. We denote the product class of test functions with respect to $\calA$ and $\calB$ as $\calC_{\calA, \calB}$, and define it as 
    \[\calC_{\calA, \calB} = \{a(\phi)b(\phi) : a \in \calA, b \in \calB\}.\]
\end{definition}

We use this notation in the following lemma, which shows that multicalibration with respect to the test functions $\calC_{\calL}$ defined in Lemma~\ref{lem:many-losses-mc} is implied by multicalibration with respect to a product class of test functions:

\begin{lemma}\label{lem:loss-to-prod-class-mc}
    Let $\calF$ be a class of loss predictors $f: \Phi \rightarrow [0, 1]$. Let $\calL$ be a class of bounded proper losses $\ell: \{0, 1\} \times [0, 1] \rightarrow [0, 1]$ with associated concave entropy functions $H_{\ell}: [0, 1] \rightarrow [0,1]$, and let $\calC_{\calL}$ be the class of test functions 
    \[\calC_{\calL} = \{(f(\phi(p, x)) - H_{\ell}(p(x)))H'_{\ell}(p(x)) : f \in \calF, \ell \in \calL\}.\]

    Denote $\calH_{\calL} = \{H_{\ell} : \ell \in \calL\}$, and $\calH'_{\calL} = \{H'_{\ell}: \ell \in \calL\}$.

    Then, 

    \[\MCE(\calC_{\calL},p) \leq \MCE(\calC_{\calF, \calH'_{\calL}}, p) + \MCE(\calC_{\calH_{\calL}, \calH'_{\calL}}, p) \leq 2\MCE(\calC_{\calF\cup \calH_{\calL}, \calH'_{\calL}}, p).\]
\end{lemma}

\begin{proof}
    By definition of $\calC_{\calL}$, we can readily decompose the multicalibration error into the desired terms:

    \begin{align*}
        \MCE(\calC_{\calL},p) &= \max_{f \in \calF, \ell \in \calL}\left|\ex[(f(\phi(p, x)) - H_{\ell}(p(x)))H'_{\ell}(p(x))(y - p(x))]\right|\\
        &\leq \max_{f \in \calF, \ell \in \calL}\left|\ex[f(\phi(p, x))H'_{\ell}(p(x))(y - p(x))]\right| + \max_{\ell \in \calL}\left|\ex[H_{\ell}(p(x))H'_{\ell}(p(x))(y - p(x))]\right|\\
        &\leq \max_{f \in \calF, \ell \in \calL}\left|\ex[f(\phi(p, x))H'_{\ell}(p(x))(y - p(x))]\right| + \max_{h \in \calH_{\calL}, h' \in \calH'_{\calL}}\left|\ex[h(p(x))h'(p(x))(y - p(x))]\right|\\
        &= \MCE(\calC_{\calF, \calH'_{\calL}}, p) + \MCE(\calC_{\calH_{\calL}, \calH'_{\calL}}, p).
    \end{align*}

    This proves the left-hand inequality. The right-hand inequality follows from the observation that \[\MCE(\calC_{\calF \cup \calH_{\calL}, \calH'_{\calL}}, p) = \max\{\MCE(\calC_{\calF, \calH'_{\calL}}, p), \MCE(\calC_{\calH_{\calL}, \calH'_{\calL}}, p)\}.\]
\end{proof}

An important property of product classes is that given two classes $\calA$ and $\calB$, whenever we have a weak learner for $\calA$ and $\calB$ is finite, we can efficiently learn a multicalibrated predictor with respect to $\calC_{\calA, \calB}$. Our approach closely follows that of~\cite{gopalan2022low}, who show how to learn multicalibrated predictors for product classes where one class depends only on $x$, and the other depends only on $p(x)$. Despite this choice in setup, their particular algorithm and results naturally generalize to the case where the two function classes can have richer input spaces. 

For completeness, we translate their algorithm and results to our setting. The algorithm for product-class multicalibration can be found in Algorithm~\ref{alg:product-class} (c.f. Algorithm 1 of~\cite{gopalan2022low}).

The following lemmas prove correctness and sample complexity of the algorithm. 

\begin{lemma}[Correctness of Algorithm~\ref{alg:product-class}]
    If Algorithm~\ref{alg:product-class} returns a predictor $p : \mathcal{X} \to [0, 1]$, then $p$ satisfies 
    \[\MCE(\calC_{\calA, \calB}, p) \leq \alpha.\]
\end{lemma}

\begin{proof}
    Observe that Algorithm~\ref{alg:product-class} only returns a predictor $p_t$ if, in the $t$th iteration, for every $b \in \calB$, the call to the weak agnostic learner $\text{WAL}_{\calA}$ returns $\bot$. By the weak agnostic learning property (Definition~\ref{def:weak-agnostic-learner}), returning $\bot$ in every call indicates that for all $b \in \calB$ and for all $a \in \calA$,
    
    \[
    \ex\left[a(\phi(p_t, x))b(\phi(p_t, x))(y - p_t(x)) \right] \leq \alpha.
    \]
    
    By definition, this means that $\MCE(\calC_{\calA, \calB}, p_t) \leq \alpha$
\end{proof}

\begin{lemma}[Termination of Algorithm~\ref{alg:product-class}]
    Algorithm~\ref{alg:product-class} is guaranteed to terminate and return a $p_T$ after $T \leq 4/\alpha^2$ iterations.
\end{lemma}

\begin{proof}
    We show that the number of iterations is bounded via a potential argument, where the potential function is the squared error of the current predictor $p$ from the bayes-optimal predictor $p^*(x) = \ex[y|x]$: $\sqLoss(p) := \ex[(p(x) - p^*(x))^2]$. 

    Note that by definition, $\sqLoss(p_0) \leq 1$, and for all $p: \calX \rightarrow [0, 1]$, $\sqLoss(p) \geq 0$. 

    The change in potential after the $t$th update can be computed as follows. Note that due to the guarantee of the weak agnostic learner, because $a_{t + 1} \neq \bot$, we are guaranteed that for the $b \in \calB$ used in the update, 
    \[\ex[b(\phi(p, x))a(\phi(p, x))(y - p(x))] \geq \alpha/2.\]

    Thus, following the same proof as Lemma~\ref{lem:improve-2}, we conclude that 
    \begin{align*}
        \sqLoss(p_t) - \sqLoss(p_{t + 1}) &\geq \ex[(p^*(x) - p_t(x))^2] - \ex[(p^*(x) - p_t(x) -\frac{\alpha}{2}\delta_{t + 1}(x))^2]\\
        &= \alpha \ex[(p^*(x) - p_t(x))\delta_{t + 1}(x)] - \frac{\alpha^2}{4}\ex[\delta_{t + 1}(x)^2]\\
        &= \alpha \ex[(p^*(x) - p_t(x))a_{t + 1}(\phi(p_t, x))b(\phi(p_t, x))] - \frac{\alpha^2}{4}\ex[a_{t + 1}(\phi(p_t, x))^2b(\phi(p_t, x))^2]\\
        &\geq \alpha^2/2 - \alpha^2/4\\
        &= \alpha^2/4.
    \end{align*}

    Thus, the potential function decreases by at least $\alpha^2/4$ in each round, and since $\sqLoss(p_0) \leq 1$ and $\sqLoss(p_t) \geq 0$ for all $t$, the total number of iterations is bounded by $T < 4/\alpha^2$.
    \end{proof}

We finally turn to the sample complexity and success probability of Algorithm~\ref{alg:product-class}. 

\begin{lemma}\label{lem:finite-B-mc-alg}
    Let $\alpha, \delta > 0$. Let $\calA \subseteq \{a: \Phi \rightarrow [-1, 1]\}$ and $\calB \subseteq \{b: \Phi \rightarrow [-1, 1]\}$ be two classes of functions, where we assume $\calA$ is closed under negation, and $\calB$ is finite. Suppose we have access to an $\alpha$-weak-agnostic-learner for $\calA$ with sample complexity $n$ and failure parameter $\beta \leq \frac{\alpha^2\delta}{4|\calB|}$. 

    Then, given $m = O(n/\alpha^2)$ samples, with probability at least $1 - \delta$ Algorithm~\ref{alg:product-class} outputs a predictor $p$ with $\MCE(\calC_{\calA, \calB}, p) \leq \alpha$. 
\end{lemma}

\begin{proof}
    We assume that we use a fresh sample for the weak agnostic learner at each iteration of size $n$. Because the algorithm terminates in at most $4/\alpha^2$ iterations, we thus need at most $4n/\alpha^2 = O(n/\alpha^2)$ fresh samples. 

    For the failure bound, note that we make at most $4|\calB|/\alpha^2$ calls to the weak agnostic learner. Via a union bound, because we assume that $\beta \leq \frac{\alpha^2\delta}{4|\calB|}$, we conclude that the probability that at least one of the calls to the weak learner fails is bounded by $\delta$, as desired.   
\end{proof}

\begin{algorithm}
    \caption{Product-Class Multicalibration}
    \begin{algorithmic}[1]\label{alg:product-class}
        \STATE \textbf{Input:} training data $\{(x_i, y_i)\}_{i=1}^m$
        \STATE \textbf{Function classes:} \\
        \quad $\calA \subseteq \{a: \Phi \to [-1, 1]\}$ (closed under negation), \\
         \quad $\calB \subseteq \{b: \Phi \to [-1, 1]\}$ (finite)
        \STATE \textbf{$\alpha$-Weak Agnostic Learner for $\calA$:} $\text{WAL}_{\calA}$
        \STATE approximation $\alpha > 0$
        \STATE \textbf{Output:} $(\calC_{\calA, \calB},\alpha)$-multicalibrated predictor $p: \mathcal{X} \to [0,1]$
        \STATE $p_0(\cdot) \gets 1/2$
        \STATE $mc \gets \text{false}$
        \STATE $t \gets 0$
        \WHILE{$\neg mc$}
            \STATE $mc \gets \text{true}$
            \FOR{each $b \in \calB$}
                \STATE $a_{t+1} \gets \text{WAL}_{\calA}(\{(\phi(p, x_i), b(\phi(p_t, x_i))(y_i - p_t(x_i)))\}_{i=1}^m)$
                \IF{$a_{t+1} = \bot$}
                    \STATE \textbf{continue}
                \ELSE
                    \STATE $\delta_{t+1}(\cdot) \gets b(\phi(p_t, \cdot)) a_{t+1}(\phi(p_t, \cdot))$
                    \STATE $p_{t+1}(\cdot) \gets \Pi_{[0,1]}(p_t(\cdot) + \frac{\alpha}{2} \delta_{t+1}(\cdot)) $ 
                    \STATE $mc \gets \text{false}$
                    \STATE $t \gets t + 1$
                    \STATE \textbf{break}
                \ENDIF
            \ENDFOR
        \ENDWHILE
        \STATE \textbf{return} $p_t$
    \end{algorithmic}
\end{algorithm}

\subsection{Multicalibration for classes with approximate bases}

In this section, we show that Algorithm~\ref{alg:product-class} can also be used to guarantee multicalibration for product classes when $\calB$ is not finite, but has a finite approximate basis. 

\begin{definition}[Finite Approximate Basis]\label{def:finite-approx-basis}
     Let $\Gamma$ be a set and $\mathcal{F} = \{f : \Gamma \to [-1, 1]\}$ a class of functions on $\Gamma$. We say that a set $\mathcal{G} = \{g : \Gamma \to [-1, 1]\}$ is a finite $\epsilon$-basis for $\mathcal{F}$ of size $d$ and coefficient norm $\lambda$, if $\calG = \{g_1, ..., g_d\}$, and for every function $f \in \mathcal{F}$, there exist coefficients $\alpha_1, \alpha_2, \ldots, \alpha_d \in [-1,1]$ satisfying
\begin{equation}
    \forall x \in \Gamma \quad \left| f(x) - \sum_{i=1}^{d} \alpha_i g_i(x) \right| \leq \epsilon \quad \text{and} \quad \sum_{i=1}^{d} |\alpha_i| \leq \lambda. 
\end{equation}
\end{definition}

\begin{lemma}\label{lem:basis-mce-bound}
    Let $\calA \subseteq \{a: \Phi \rightarrow [-1, 1]\}$ and $\calB \subseteq \{b: \Phi \rightarrow [-1, 1]\}$ be two classes of functions. Suppose that $\calB$ has a finite approximate $\epsilon$-basis $\calG$ with coefficient norm $\lambda$. Then, for any predictor $p: \calX \rightarrow [0, 1]$, 
    \[\MCE(\calC_{\calA, \calB}, p) \leq \lambda\MCE(\calC_{\calA, \calG}, p) + \epsilon\]
\end{lemma}

\begin{proof}
    The upper bound quickly follows from expanding the definition of $\MCE(\calC_{\calA, \calB}, p)$. Note that 

    \begin{align*}
        \MCE(\calC_{\calA, \calB}, p) &= \max_{a \in \calA, b \in \calB}\left|\ex[a(\phi(p, x))b(\phi(p, x))(y - p(x))]\right|\\
        &\leq \max_{a \in \calA, b \in \calB}\left|\ex[a(\phi(p, x))\left(\sum_{i = 1}^d g_i(\phi(p, x))\alpha_i(b)\right)(y - p(x))]\right| + \epsilon\\
        &\leq \max_{a \in \calA, g \in \calG}\lambda\left|\ex[a(\phi(p, x))g(\phi(p, x))(y - p(x))]\right| + \epsilon\\
        &= \lambda\MCE(\calC_{\calA, \calG}, p) + \epsilon
    \end{align*}
\end{proof}

Thus, we get the following immediate Corollary of Lemmas~\ref{lem:finite-B-mc-alg} and \ref{lem:basis-mce-bound}:

\begin{corollary}\label{cor:approx-basis-alg}
    Let $\alpha, \delta > 0$. Let $\calA \subseteq \{a: \Phi \rightarrow [-1, 1]\}$ and $\calB \subseteq \{b: \Phi \rightarrow [-1, 1]\}$ be two classes of functions, where we assume $\calA$ is closed under negation, and $\calB$ has a finite approximate $\epsilon$-basis of size $d$ and coefficient norm $\lambda$. Suppose we have access to an $\alpha$-weak-agnostic-learner for $\calA$ with sample complexity $n$ and failure parameter $\beta \leq \frac{\alpha^2\delta}{4d}$. 

    Then, there exists an algorithm that, given $m = O(n/\alpha^2)$ samples, with probability at least $1 - \delta$ outputs a predictor $p$ with $\MCE(\calC_{\calA, \calB}, p) \leq \lambda\alpha + \epsilon$. 
\end{corollary}

\subsection{Instantiating multicalibration for loss prediction}

We are finally ready to show that we can use multicalibration to learn a predictor with accurate self-entropy predictions for any loss in a rich class. 

\begin{theorem}\label{thm:general-approx-basis-mc}
    Fix $\delta \geq 0$. Let $\calL$ be some class of bounded proper losses, and let $\calF$ be a class of loss predictors $\calF: \Phi \rightarrow [-1, 1]$ that is closed under negation and contains the class of self entropy predictors, $\calH_{\calL} = \{H_{\ell}\}_{\ell \in \calL}$. Suppose that the class of functions $\calH'_{\calL} =\{H'_{\ell}\}_{\ell \in \calL}$ has a finite approximate $\epsilon$-basis of size $d$ and coefficient norm $\lambda$. Further assume that we have access to an $\alpha$-weak-agnostic-learner for $\calF$ with sample complexity $n$ and failure parameter $\beta \leq \frac{\alpha^2\delta}{4d}$. 

    Then, there exists an algorithm that, given $m = O(n/\alpha^2)$ samples, with probability at least $1 - \delta$ outputs a predictor $p$ such that 
    \[\max_{\ell \in \calL} \max_{\lossPred \in \calF} \adv(\lossPred) \leq 4\lambda\alpha + 4\epsilon.\]
\end{theorem}

\begin{proof}
    The proof combines the helper lemmas we have proved in this section. 

    First, note that by Lemma~\ref{lem:many-losses-mc}, we are guaranteed that 
    \[\max_{\ell \in \calL} \max_{\lossPred \in \calF} \adv(\lossPred) \leq 2\MCE(\calC_{\calL}, p),\]
    where $\calC_{\calL}$ is defined as in Lemma~\ref{lem:many-losses-mc}. 

    By Lemma~\ref{lem:loss-to-prod-class-mc}, we can further guarantee that 
    \[\max_{\ell \in \calL} \max_{\lossPred \in \calF} \adv(\lossPred) \leq 4\MCE(\calC_{\calF \cup \calH_{\calL}, \calH'_{\calL}}, p) = 4\MCE(\calC_{\calF, \calH'_{\calL}}, p),\]
    where the last equality follows from our assumption that $\calF$ contains $\calH_{\calL}$.

    From here, we can now apply the result of Corollary~\ref{cor:approx-basis-alg}, which guarantees that given $m = O(n/\alpha^2)$ samples, we can output a predictor satisfying 
    $\MCE(\calC_{\calF, \calH'_{\calL}}, p) \leq \lambda\alpha + \epsilon$. 

    Substituting this bound into our upper bound, we conclude that we get a predictor satisfying 
    \[\max_{\ell \in \calL} \max_{\lossPred \in \calF} \adv(\lossPred) \leq 4\lambda\alpha + 4\epsilon.\]

    This completes the proof. 
\end{proof}

\subsection{Proof of Theorem~\ref{thm:1-lip-mc-pred}}\label{sec:1-lip-mc-pred-pf}

Theorem~\ref{thm:1-lip-mc-pred} instantiates the general result of Theorem~\ref{thm:general-approx-basis-mc} for a class $\calL$ that has a finite approximate basis: the class of all 1-Lipschitz proper losses. 

We appeal to a result proved by~\cite{OKK25}, which proves the existence of such a basis for this class:

\begin{lemma}[\cite{OKK25}, Lemma 5.4]\label{lem:1-lip-basis}
    Let $\epsilon > 0$, and let $\calL_{Lip}$ be the class of proper 1-Lipschitz loss functions $\ell: \{0, 1\} \times [0, 1] \rightarrow [0,1]$. The class $\{H'_{\ell}\}_{\ell \in \calL_{Lip}}$ admits an $\epsilon$-approximate basis of size $\lceil \frac{2}{\epsilon} + 1 \rceil$ and coefficient norm 4. 
\end{lemma}

with this Lemma in hand, we are ready to prove the theorem. 

\begin{proof}[Proof of Theorem~\ref{thm:1-lip-mc-pred}]
    We instantiate the basis from Lemma~\ref{lem:1-lip-basis}, and use this as the input to Theorem~\ref{thm:general-approx-basis-mc}.

    Because $d = \lceil 2/\epsilon + 1\rceil$, the bound on the weak-agnostic-learner's error probability becomes $\beta \leq \frac{\alpha^2\delta}{4\lceil 2/\epsilon + 1\rceil}$, and the resulting error guarantee gives us a predictor $p$ satisfying 

    \[\max_{\ell \in \calL_{Lip}} \max_{\lossPred \in \calF} \adv(\lossPred) \leq 16\alpha + 4\epsilon\]
    with probability at least $1 - \delta$. 
\end{proof}

\end{document}